\documentclass[journal]{IEEEtran}
\usepackage{times}
\usepackage{epsfig}
\usepackage{graphicx}
\usepackage{amsthm}
\usepackage{amsmath}
\usepackage{amssymb}
\usepackage{amsfonts,amsthm}
\usepackage{graphicx,subfigure,booktabs,multirow}
\usepackage{mathrsfs,enumerate,color}%,mathdots
\usepackage{algorithm,algpseudocode}
\usepackage{hyperref}
\usepackage{graphicx}
\usepackage{epstopdf}
\usepackage{float}
\usepackage[sort&compress, square, numbers]{natbib}

\usepackage{diagbox}
\usepackage{threeparttable}
\usepackage{booktabs}
\usepackage{caption2}

\newtheorem{theorem}{Theorem}[section]

\newtheorem{remark}{Remark}[section]
\newtheorem{definition}{Definition}[section]
\newtheorem{example}{Example}[section]

\newcommand{\iq}{{\bf i}}
\newcommand{\jq}{{\bf j}}
\newcommand{\kq}{{\bf k}}

\newcommand{\Aq}{{\bf A}}
\newcommand{\aq}{{\bf a}}

\newcommand{\Cq}{{\bf C}}

\newcommand{\Vq}{{\bf V}}
\newcommand{\vq}{{\bf v}}
\newcommand{\Uq}{{\bf U}}
\newcommand{\uq}{{\bf u}}

\newcommand{\Lq}{{\bf L}}

\newcommand{\Sq}{{\bf S}}

\newcommand{\Xq}{{\bf X}}
\newcommand{\xq}{{\bf x}}
\newcommand{\Zq}{{\bf Z}}

\newcommand{\Yq}{{\bf Y}}

\newcommand{\wq}{{\bf w}}

\newcommand{\Fq}{{\bf F}}

\newcommand{\Pq}{{\bf P}}
\newcommand{\Qqq}{{\bf Q}}

\newcommand{\Qq}{\mathbb{Q}}
\newcommand{\XXq}{\boldsymbol{\mathcal{X}}}
\newcommand{\LLq}{\boldsymbol{\mathcal{L}}}
\newcommand{\SSq}{\boldsymbol{\mathcal{S}}}
\newcommand{\TTq}{\boldsymbol{\mathcal{T}}}
\newcommand{\YYq}{\boldsymbol{\mathcal{Y}}}
\newcommand{\ZZq}{\boldsymbol{\mathcal{Z}}}
\newcommand{\PPq}{\boldsymbol{\mathcal{P}}}
\newcommand{\QQq}{\boldsymbol{\mathcal{Q}}}

\newcommand{\tr}{\operatorname{Tr}}

\newcommand{\amin}{\operatorname{arg~min}}

\ifCLASSINFOpdf
\else
\fi
\begin{document}
	\title{A New Low-Rank Learning Robust Quaternion Tensor Completion Method for Color Video Inpainting Problem and Fast Algorithms}

		\author{Zhi-Gang~Jia \thanks{Z. Jia is with the Research Institute of Mathematical Science and the School of Mathematics and Statistics, Jiangsu Normal University,
				Xuzhou 221116, P. R. China (e-mail: zhgjia@jsnu.edu.cn)}
		 and
			Jing-Fei~Zhu \thanks{J. Zhu is with the School of Mathematics and Statistics, Jiangsu Normal University, Xuzhou 221116, P. R. China.}
%			Michael K. Ng,
%			and~Xile~Zhao% <-this % stops a space
			% <-this % stops a space
		}

	\maketitle
	
		\begin{abstract}
		The color video inpainting problem is one of the most challenging problem in the modern imaging science. It aims to recover a color video from a small part of pixels that may contain noise. However, there are less of robust models that can simultaneously preserve the coupling  of color channels and the evolution of color video frames. In this paper, we present a new robust quaternion tensor completion (RQTC) model to solve this challenging problem and derive the exact recovery theory. The main idea is to build a quaternion tensor optimization model to recover a low-rank quaternion tensor that represents the targeted  color video and a sparse quaternion tensor that represents noise. This new model is very efficient to recover high dimensional data  that satisfies the prior low-rank assumption. To solve the case without low-rank property, we introduce a new low-rank learning RQTC model, which rearranges  similar  patches classified by a quaternion learning method into smaller tensors satisfying the prior low-rank assumption.  We also propose fast algorithms  with global convergence guarantees. In numerical experiments, the proposed methods successfully recover   color videos with eliminating color contamination and keeping the continuity of video scenery, and their solutions are of higher quality in terms of PSNR and SSIM values than the state-of-the-art algorithms.
	\end{abstract}
	

	% \fi
	\IEEEpeerreviewmaketitle

	\section{Introduction}
	%%%\IEEEPARstart{T}{his} demo file is intended to serve as a ``starter file''
	%%%for IEEE journal papers produced under \LaTeX\ using
	%%%IEEEtran.cls version 1.8b and later.
	%%%% You must have at least 2 lines in the paragraph with the drop letter
	%%%% (should never be an issue)
	%%%I wish you the best of success.
	%%%
	%%%\hfill mds
	%%%
	%%%\hfill August 26, 2015
	%%%
	%%%\subsection{Subsection Heading Here}
	%%%Subsection text here.
	%%%
	%%%% needed in second column of first page if using \IEEEpubid
	%%%%\IEEEpubidadjcol
	%%%
	%%%\subsubsection{Subsubsection Heading Here}
	%%%Subsubsection text here.

		\IEEEPARstart{M}{any} applications of multi-dimensional data (tensor data) are becoming popular. For instance, color videos or images can be seen as 3-mode or 2-mode quaternion data. With its capacity to capture the fundamental substructures and color information, quaternion tensor-based modeling is an obvious choice to solve color video processing problems. A modern and challenging problem is color video inpainting, which aims to recover a color video from a sampling of its pixels that may contain noise. In mathematical language, this problem  is robust quaternion tensor completion (RQTC) problem. There are currently less of methods to solve this problem because it is difficult to   preserve the coupling  of color channels and the evolution of color video frames. In this paper, we present new robust quaternion tensor completion (RQTC) models to solve this challenging problem.
		
	%	of the prior assumption that 
	%	the targetted tensor is of  low-rank property.
		
%		is called quaternion tensor completion (QTC) problem. Recovering the incomplete entries and extracts the latent low-rank component in observed data is called
%	    robust quaternion tensor completion (RQTC) problem. RQTC can be regarded as a generalized form of quaternion tensor completion and  robust quaternion tensor principal component analysis. Although exist in high-dimensional spaces,
%	the tensor of interest is frequently low rank \cite{kolda2009}.

	For a single color image, the robust quaternion matrix completion (RQMC) method proposed in \cite{jia2018quaternion}  theoretically solved the color image inpainting under the  incoherence conditions. Chen and Ng \cite{errorchen2022} proposed a cross-channel weight strategy and analysed the error
	bound of RQMC problem.
    Xu \textit{et al.} \cite{xkscz2022} proposed a new model to combine deep prior and low-rank quaternion prior in color image processing.
    A generous amount of practical applications indicate that RQMC can completely recover the color images of which low-frequency 
	information dominates, but it fails to recover the color image of which high-frequency information dominates.
	To inpaint color images in the latter case, a new nonlocal self-similarity (NSS) based RQMC was introduced in \cite{jia2022quaternion} to compute an optimal  approximation to the  color image. The main idea is to gather similar patches into several color images of small size that mainly contain low-frequency information. This NSS-based RQMC  uses the distance function to find low-rank structures of color images. It is also applied to solve color video inpainting problems and achieves color videos of high quality. However, it overlooks the global information that reflects the potential relation of continuous frames. 
	So we need to build a quaternion tensor-based model for color video inpainting.
	
%	In addition, an algorithm using subspace clustering for high-rank matrix completionhas been proposed in \cite{blr2012highrank},  which is 
%	under the assumption that
%	the columns of the matrix belong to a union of multiple low-rank subspaces.
	Recall that several famous real tensor decompositions  \cite{kolda2009} serve as the foundation  of modern robust tensor completion (RTC) approaches. For instance, Liu \textit{et al.} \cite{liu2009} presented a sum-of-nuclear-norms (SNN) as tensor rank in the RTC model. This representation  depends on the Tucker decomposition \cite{Tucker1966} and SNN model is proved with an exact recovery guarantee in \cite{huang2014}.   Gao and Zhang \cite{gh2023} proposed a novel  nonconvex model with $\ell_p$ norm to solve RTC problem. Jiang \textit{et al.} \cite{jzzn2023}  presented  a data-adaptive dictionary to determine relevant third-order tensor tubes  and established a new tensor learning and coding model. Ng \textit{et al.} \cite{nzz2020}  proposed a novel unitary transform method that is very efficient by using similar patches strategy to form a third-order sub-tensor.  Wang \textit{et al.} \cite{wzj2022} recovered tensors by two new tensor norms. Zhao \textit{et al.} \cite{xmdl2022} proposed an equivalent  nonconvex  surrogates of RTC problem and analysed the recovery error bound. These RTC methods 
%	[\cite{hllz2023},\cite{jzzn2023},\cite{lqz2023}-\cite{wzj2022},\cite{xmdl2022}]  
	have been successfully applied in color image or video processing. 
	However, RTC models  regard color images as 3-mode real tensors \cite{lqz2023, lzjynl2022} and color videos as 4-mode real tensors \cite{hllz2023} and thus, they usually independently process three color channels and ignore the mutual connection among channels. 
	%and some RTC models are based on 3-mode tensor decomposition which cannot deal with the color video processing.
	
	Since quaternion has $\iq,\jq,\kq$ three imaginary parts, a color pixel can be seen as a pure quaternion.  Based on quaternion representation and calculation, the color information can be preserved in the color image processing. A low-rank quaternion tensor completion method (LRQTC) was proposed in \cite{miao2020qtc}. It cannot deal with the noisy or corrupted problem since it only contains  a low-rank regularization term.
%	A new TNSS-based
%	QMC method has been proposed in \cite{jia2022quaternion} and is developed to inpaint color videos from
%	incomplete and corrupted frames. It uses distance function to find low rank structures of color video. This provides an idea to solve the RQTC problem.
%	
%	In this paper, we focus on low rank quaternion tensor estimation under partial or corrupted observations.
	 By introducing a new sparsity regularization term into the subject function, we propose a new RQTC method for color video inpainting with missing and corrupted pixels. There are  two new models. One is a robust quaternion tensor completion (RQTC) model, which recovers color videos from a  global view. It is essentially in the form of a quaternion minimization problem with rank and $\ell_1$ norm two regularization terms.
%	The quaternion tensor nuclear norm,  we adopt SNN as a tractable measure of the tensor rank in practical settings and use  alternating direction method of multipliers (ADMM) \cite{boyd2011distributed} framework to solve this minimization problem.
	The other is a low-rank learning RQTC (LRL-RQTC) model. We intend to learn similar information to form low-rank structure and prove the numerical low-rank property in theory. Under the view of numerical linear algebra, the principal components computed by two-dimensional principal  component analysis (2DPCA) \cite{jda20042dpca} span an optimal subspace on which projected samples are maximally scattered. Meanwhile, low-rank approximations of original samples can be simultaneously reconstructed from such low-dimension projections. Recently, 2DPCA is generalized to quaternion, named by two-dimensional quaternion principal  component analysis (2DQPCA), in \cite{jia20172dqpca}  and 2DQPCA performs well on color image clustering. 2DQPCA and the variations extract features from training data and utilize these features to project training and testing samples into projections of low dimensions for efficient use of available computational resources.  We find that 2DQPCA is a good learning method to extract low-rank structure from quaternion tensors.	So  we apply 2DQPCA method to learn the low-rank structure adaptively. 
%	Both methods based on nuclear norm shrinks all singular values equally and calculate low rank approximation of quaternion singular. During this process, we need to calculate the singular triplets of matrices. In \cite{jia2022quaternion}, we use Lanczos bidiagonalization process to estimate approximate image. Inspired by block-form Lanczos process \cite{br2006irlbablk} \cite{jlz2022block}, we proposed implicitly restarted block Lanczos bidiagonalization algorithm to get higer precision and faster computational speed.
	%%
	
	The highlights are as follows:
	\begin{itemize}
		\item  We present a novel RQTC method for color video inpainting problem with missing and corrupted  pixels and derive  the exact recovery theorem. 
		This method can simultaneously preserve the coupling  of color channels and the evolution of color video frames.
		\item  We firstly introduce  the 2DQPCA technology into color video inpainting to learn the low-rank structures of quaternion tensors
		  and present a new low-rank learning RQTC model.  Moreover,  the numerical low-rank property is proved in theory.
		% firstly introduce the 2DQPCA technology into color video inpainting to learn the low rank structures of quaternion tensors. A new LRL-RQTC method is presented and the numerical low rank property is proved in theory.
%		A new LRL-RQTC method is proposed and successfully shows the effectiveness of restoring  color videos from missing and noisy entries. 
		\item  We  design new RQTC and LRL-RQTC algorithms based on the  alternating direction method of multipliers  (ADMM) framework and apply them to solve color video inpainting problems with missing or noisy pixels. The color videos  computed by the newly proposed algorithms are of higher quality in terms of PSNR and SSIM values than those by  the state-of-the-art algorithms.
	\end{itemize}

This paper is organized as follows. 
In Section \ref{sec:pre}, we introduce  preliminaries about quaternion matrix and the  quaternion matrix completion method.
In Section \ref{sec:main}, we present new  robust quaternion tensor completion and low-rank learning robust quaternion tensor completion models, including solving procedure,  sufficient conditions for precise recovery, convergence analysis, 2DQPCA-based classification technology to learn low-rank information, and  theoretical analysis of  numerical low-rank.
In Section \ref{sec:ex}, we propose experimental results of color video inpainting, which indicate the advantages of the newly proposed methods on quality of restorations.
In Section \ref{sec:con}, we conclude the paper and present prospects.

	\section{Preliminaries}\label{sec:pre}
   Several necessary results about quaternion matrices are recalled in this section. 
	\subsection{Quaternion matrix}
	Let $\mathbb{Q}$ denotes the set of quaternion and a quaternion $\aq$  has one real part $a_0 \in \mathbb{R}$, three imaginary parts $a_1,a_2,a_3\in \mathbb{R}$ and is expressed as
	$\aq=a_0+a_1\iq+a_2\jq+a_3\kq,      \iq^2=\jq^2=\kq^2=\iq\jq\kq=-1$ \cite{Hamilton1866}. 
	A symbol of boldface is used to express quaternion scalar, vector, matrix or tensor.
	A quaternion matrix $\Aq=A_0+A_1\iq+A_2\jq+A_3\kq \in \mathbb{Q}^{m\times n}$ with $A_0, \ldots, A_3 \in \mathbb{R}^{m\times n}$.  If $A_0=0$ and $A_1, A_2, A_3 \neq 0$, $\Aq$ is named by a purely imaginary quaternion matrix. 
The quaternion shrinkage function $\texttt{shrinkQ}$ is defined in \cite{jia2022quaternion} by:
	\begin{equation}\label{e:shrinkQ}
		\begin{aligned}
			\texttt{shrinkQ}(\Aq,\tau)
			&=\underset{\Aq}{\amin}~~ \frac{1}{2}\|\Aq-\Zq\|_F^2+\tau\left\Vert\Aq \right \Vert_{1},\\
			&=\![\texttt{signQ}(\aq_{ij})\max(\texttt{absQ}(\aq_{ij})-\tau,0)]
		\end{aligned} 
	\end{equation}
	where $\tau > 0$, $\texttt{absQ}(\Aq):=[|\aq_{ij}|]$ 
	and
	$$\texttt{signQ}(\aq_{ij}):=\left\{\begin{array}{ll}
		\aq_{ij}/|\aq_{ij}|, & \text{if}~~|\aq_{ij}|\ne 0;\\
		0, &\text{otherwise.}
	\end{array}\right.$$
	
	Suppose ${\bf A} = {\bf U}  \Sigma {\bf V}^*$ 
     is the singular value decomposition 
	and denote $\{\sigma_{j},\uq_{j},\vq_{j}\}$ by the  singular triplets of a  quaternion matrix $\Aq\in\mathbb{Q}^{m\times n}$. 
	The quaternion singular thresholding  function $\texttt{approxQ}$ is defined in \cite{jia2022quaternion}  by:
	\begin{equation}\label{e:approxq}
		\begin{aligned}
		\texttt{approxQ}(\Aq,\tau)
		&=\underset{\Aq}{\amin}~~
		(\left\Vert  {\Aq} \right\Vert_{\ast}+\frac{1}{2\tau}\|\Aq-\Yq\|_F^2)\\
		&=\Uq\texttt{diag}(\sigma_1,\cdots,\sigma_k,0,\cdots,0)\Vq^*, \tau>0,
		\end{aligned}
	\end{equation}
	where $\sigma_1\ge\cdots\ge\sigma_k>\tau$ and the singular values $\sigma_j <\tau$ are substituted by zeros. 
Quaternion matrix norms are defined by
	$\left\Vert \Aq\right\Vert _{1} :=\sum\limits_{i=1}^{m}\sum\limits_{j=1}^{n}\left\vert
	\aq_{ij}\right\vert$, 
		$\left\Vert \Aq\right\Vert _{\infty } :=\max\limits_{i,j}\left\vert \aq_{ij}\right\vert$,
	$\left\Vert \Aq\right\Vert _{F}=\sqrt{\sum\limits_{i=1}^{m}\sum\limits_{j=1}^{n}\left\vert \aq_{ij}\right\vert ^{2}} :=\sqrt{ \tr\left( \Aq^*\Aq\right) }$,
and
	$\|{\bf A} \|_{*} := \sum\limits_{i=1}^r \sigma_i$.

	\subsection{Robust quaternion matrix completion method}
	A low-rank quaternion matrix $\Lq_0$ can be recovered completely from an observed quaternion matrix
	$\Xq=\mathcal{P}_{\Omega}(\Lq_{0}+\Sq_{0})$  by  the RQMC method \cite{jia2018quaternion}, where $\Sq_0$  is a noisy matrix and $\mathcal{P}_{\Omega} $ is a random sampling operator:
	$$
	\mathcal{P}_{\Omega} ({\bf X}) = \left \{ \begin{array}{cc}
		{\bf x}_{i,j},& (i,j)\in \Omega,  \\
		0,  &\text{otherwise}.
	\end{array}
	\right.
	$$
	By \cite[Theorem 2]{jia2018quaternion},  if the sufficient conditions are satisfies, $\Lq_0$ can be exactly computed by solving  the following minimization problem with $\lambda = \frac{1}{\rho n_{(1)}}$,
	\begin{equation} \label{model}
		\begin{array}{rl}
			\underset{\bf L, \bf S}{\min}&
			\left\Vert {\bf L} \right \Vert_{\ast} + \lambda \left\Vert {\bf S}\right\Vert _{1}\\
			\ {\rm s.t. } &
			\mathcal{P}_{\Omega}\left({\bf L}+{\bf S}\right)={\bf X}.
		\end{array}
	\end{equation}
A practical QMC algorithm is given in \cite[Supplementary Material]{jia2022quaternion}.  
%Unfortunately,  its convergence theory is missed in the literature. 
The augmented Lagrangian function of \eqref{model} is defined by
	\begin{equation*} \label{ucmodel}
		\begin{array}{rl}
			\underset{{\bf L}, {\bf S}, {\bf P}, {\bf Q}}{\min}&
			\left\Vert {\bf L} \right \Vert_{\ast} + \lambda \left\Vert {\bf S}\right\Vert _{1}+\frac{\mu}{2} \left\Vert {\bf L-\bf P+\bf Y/\mu}\right\Vert _{F}^2\\
			&+\frac{\mu}{2} \left\Vert {\bf S-\bf Q+\bf Z/\mu}\right\Vert _{F}^2\\
			\ {\rm s.t.} &
			\mathcal{P}_{\Omega}\left({\bf P}+{\bf Q}\right)={\bf X},
		\end{array}
	\end{equation*}
	where  $\mu$ is the penalty parameter.
 The solving procedure is
	\begin{equation}\label{a:admmqmc}
	\left\{
		\begin{aligned}
			\Lq^{t+1} &= \texttt{approxQ}(\Pq^{t}-(\Yq/\mu)^{t},\frac{1}{\mu}),\\
			\Qqq_{i,j}^{t+1} &= \left \{ \begin{array}{cc}
				(\Xq^{t}-\Pq^{t})_{i,j},& (i,j)\in \Omega,\\
				(\Sq^{t}+\Zq^{t}/\mu)_{i,j},& \text{otherwise},
			\end{array}
			\right.\\	
			\Sq^{t+1} &= \texttt{shrinkQ}(\Qqq^{t}-(\Zq/\mu)^{t},\frac{\lambda}{\mu}), \\
			\Fq_{i,j}^{t+1} &=(\mu\Lq^{t}+\mu\Xq^{t}-\mu\Sq^{t}+\Yq^{t}-\Zq^{t})_{i,j},\\
			\Pq_{i,j}^{t+1}  &= \left \{ \begin{array}{ll}
				\Fq_{i,j}^{t+1}/2\mu, &(i,j)\in \Omega,\\
				(\Lq^{t}+\Yq^{t}/\mu)_{i,j},& \text{otherwise},
			\end{array}
			\right.\\
			\Yq^{t+1} &= \Yq^{t} + \mu(\Lq^{t}-\Pq^{t}),\\
			\Zq^{t+1} &= \Zq^{t} + \mu(\Sq^{t}-\Qqq^{t}).
		\end{aligned}
		\right.
	\end{equation}
%	More specifically, we choose an initial $(\Pq^{0}, \Qqq^{0}, \Yq^{0},\Zq^{0})$, a suitable $\lambda,
%	\mu$, maximum iteration
%	number $N$, and error tolerance $\epsilon$ as input, then update $(\Pq^{t}, \Qqq^{t},\Lq^{t},\Sq^{t},  \Yq^{t},\Zq^{t})$
%	to $(\Pq^{t+1}, \Qqq^{t+1},\Lq^{t+1},\Sq^{t+1}, \Yq^{t+1},\Zq^{t+1})$ until the stopping criterion $\|\Xq - \Lq -\Sq\|_F \le \epsilon$ is satisfied,
%	or the number of iteration reaches $N$. 

			\section{Robust Quaternion Tensor Completion Models and  Fast Algorithms} \label{sec:main}
	In this section, we propose two new RQTC models to solve color video inpainting problem with partial and corrupted pixels, as well as their fast algorithms. 
	% and an acceleration strategy is also introduced  proposed to reduce the computational time of QMC algorithm.
	
		The boldface Euler script letters, e.g. $\XXq$, $\LLq$ and $\SSq$, are used to denote quaternion tensors.  Let  $\XXq \in \Qq^{n_1\times n_2 \times \cdots \times n_k}$ be a $k$-mode  quaternion tensor.  The elements of $\XXq$ are denoted by $\xq_{i_1i_2\cdots i_k}$, where $1\le i_j\le n_j,j=1,\cdots,k$. A $j$-mode fiber is an $n_j$-dimensional column vector constructed by entries with fixing all indexes except the $j$th one, denoted by $\XXq(i_1,\cdots,i_{j-1},:,i_{j+1},\cdots, i_k)$.  The number of  $j$-mode fibers is $\prod\limits_{i\neq j}n_i$. Concatenate all of $j$-mode fibers  as column vectors (in dictionary order) into a quaternion matrix  $\Xq_{(j)}\in \Qq^{n_j \times \prod \limits_{i\neq j}n_i}$ and name it by the $j$-mode unfolding of  quaternion tensor $\XXq$. 
We define the `unfold' function on quaternion tensor $\XXq$ by
	\begin{equation}\label{f:unfold}
		 \texttt{unfold}_j(\XXq):=\Xq_{(j)},~~j=1,\cdots,k,
	\end{equation}
and the `fold' function  by 
\begin{equation}\label{f:fold} \texttt{fold}_j(\Xq_{(j)}):=\XXq.\end{equation}
A slice of $\XXq$ is a quaternion matrix of the form $\XXq(i_1,\cdots,i_{j_1-1},:,i_{j_1+1},\cdots,i_{j_2-1},:,i_{j_2+1},\cdots, i_k)$ with all the indexes being fixed except  $j_1$ and $j_2$.
	
One important application of quaternion tensor is color video processing.  A color video with $n_3$ frames can be seen as a $3$-mode quaternion tensor $\XXq = \mathcal{R}{\bf i} + \mathcal{G} {\bf j} + \mathcal{B}{\bf k} \in {\mathbb Q}^{n_1 \times n_2 \times n_3}$, where  $\mathcal{R}$, $\mathcal{G}$ and $\mathcal{B} \in \mathbb{R}^{n_1 \times n_2 \times n_3}$ represent the red, green and blue three channels.  Mathematically, color video inpainting problem with noise is exactly the RQTC problem  (with $k=3$),  which will be characterized later in \eqref{m:qtc}.

	\subsection{RQTC model}
	 Let $\XXq\in {\mathbb Q}^{n_1 \times n_2 \times\cdots \times n_k}$ denotes observed quaternion tensor with missing and/or corrupted entries.  Then the RQTC problem is mathematically modeled by the following minimization problem,
	\begin{equation} \label{m:qtc}
		\begin{array}{rl}
			\underset{\LLq, \SSq}{\min}&
			\left\Vert {\LLq} \right \Vert_{\ast} + \lambda \left\Vert {\SSq}\right\Vert _{1}\\
			\ {\rm s.t. } &
			\mathcal{P}_{\Omega}\left({\LLq}+{\SSq}\right)=\XXq
		\end{array}
	\end{equation}
	where   $\LLq$, $\SSq$ $\in {\mathbb Q}^{n_1 \times n_2 \times\cdots \times n_k}$ denote the target low-rank tensor and the sparse data,
	respectively.  
	The quaternion tensor random sampling operator  $\mathcal{P}_{\Omega}$ is defined by
	$$(\mathcal{P}_{\Omega}[\XXq])_{i_1i_2\cdots i_k} :=\left\{
	\begin{aligned}
		\xq_{i_1i_2\cdots i_k}&, \quad(i_1,i_2,\cdots i_k)\in \Omega\\
	    0&,\quad \text{otherwise}.
	\end{aligned}
     \right.
	$$
The quaternion tensor nuclear norm and the $\ell_1$ norm are defined by
	\begin{equation}\label{snn}
		\left\Vert {\LLq} \right \Vert_{\ast} := \sum\limits_{j=1}^{k}\alpha_j\left\Vert  {\Lq_{(j)}} \right \Vert_{\ast},~~
		\left\Vert {\SSq} \right \Vert_{1} :=\left\Vert  {\Sq_{(1)}} \right \Vert_{1}
	\end{equation}
	where $\alpha_j$'s  are constants satisfying $\alpha_j \ge 0$ and $\sum_{j=1}^{k}\alpha_j=1$. 
Here, the quaternion tensor nuclear norm is essentially a convex combination of quaternion matrix nuclear norms which are generated by unfolding the tensor along each mode. Notice that $\left\Vert  {\Sq_{(1)}} \right \Vert_{1}=\left\Vert  {\Sq_{(2)}} \right \Vert_{1}=\cdots=\left\Vert  {\Sq_{(k)}} \right \Vert_{1}$.

	%Based on the previous concepts, we translate the problem \eqref{m:qtc} into the following robust low rank quaternion tensor completion model:
	%\begin{equation} \label{m:qtc2}
	%	\begin{array}{rl}
	%	\underset{\Lq_{(p)}, \Sq_{(p)}}{\min}&
	%		\sum\limits_{p=1}^{3}\alpha_p\left\Vert  {\Lq_{(p)}} \right \Vert_{\ast} + \lambda \left\Vert  {\Sq_{(1)}} \right \Vert_{1}\\
	%	\ {\rm s.t. } &
	%		\mathcal{P}_{\Omega}\left({\mathcal{L}}+{\mathcal{S}}\right)={\mathcal{X}}
	%	\end{array}
	%\end{equation}

To derive the exact RQTC theorem, we need to build the incoherence conditions of quaternion tensors.
	
%	Hence, we propose the following set of incoherence conditions for a quaternion tensor $\XXq$  by extending the quaternion matrix incoherence conditions \cite{jia2018quaternion} to the unfoldings of $\XXq$.
	\begin{definition}\label{d:inco}
		For a quaternion tensor $\XXq\in \mathbb{Q}^{n_1 \times n_2 \times\cdots \times n_k}$, suppose  each $\Xq_{(j)}$ has the singular value decomposition
		$$\Xq_{(j)}=\Uq_j\Sigma_j\Vq_j^*, j=1,2,\cdots,k.
		$$
		Let
		\begin{equation*}
		r_j = {\rm rank}(\Xq_{(j)}) \quad and \quad \TTq:=\sum\limits_{j=1}^{k}\sqrt{n_j^{(1)}}   \texttt{fold}_j(\Uq_j\Vq_j^*).
		\end{equation*}
		Then the conditions of quaternion tensor incoherence with $\mu$, $n_j^{(1)}=\max\left(n_j,\prod_{i\neq j}n_i\right)$, $n_j^{(2)}=\min\left(n_j,\prod_{i\neq j}n_i\right)$ are as follows:
		
%		that for  a $j$-mode of $\XXq$ such that:
	 (1) $j$-mode incoherence
		\begin{equation}\label{e:inco}
			\begin{split}
		 \underset{j}{\max}\|\Uq_j^*e_i\|^2 \le \frac{\mu r_j}{n_j},~\underset{j}{\max}\|\Vq_j^*e_i\|^2 \le \frac{\mu r_j}{\prod_{i=1,i\neq j}^{k}n_i},\\
		\|\Uq_j\Vq_j^*\|_{\infty} \le \mu\sqrt{\frac{ r_j}{n_j^{(1)}n_j^{(2)}}},
	\end{split}	
	\end{equation}

	(2) mutual incoherence
		 \begin{equation}\label{e:inco2}
		 \frac{\|\TTq\|}{k} \le \mu\sqrt{\frac{r_j}{n_j^{(2)}}}.
		 	 \end{equation}
%		where $n_j^{(1)}=\max\left(n_j,\prod_{k\neq j}n_k\right)$, $n_j^{(2)}=\min\left(n_j,\prod_{k\neq j}n_k\right),\mu$ is a parameter.
	\end{definition}

	The condition \eqref{e:inco2} strengthens the original quaternion matrix incoherence condition \eqref{e:inco} and keeps balance between the ranks of quaternion $\XXq$. Indeed, define $\kappa_j:=\frac{r_j}{n_j^{(2)}}$ and $\kappa:=\max{\kappa_j}$. Clearly, a larger $\kappa$ means that even though quaternion tensor $\XXq$ has a certain mode  of low-rank but also has a mode of high rank. 
	
%	Under \eqref{e:inco} and \eqref{e:inco2}, the existence of only one low rank mode is sufficient for exact recovery and we do not know in advance which mode is of low rank.
	\begin{theorem}\label{t:exactRQTC}
		Suppose a quaternion tensor $\LLq_0 \in \mathbb{Q}^{n_1\times n_2\times\cdots\times n_k}$ meets the incoherence conditions  in definition \eqref{d:inco},
		a set $\Omega$ is uniformly distributed with cardinality $m = \rho n_j^{(1)}n_j^{(2)}$ and each
		observed entry is  corrupted with probability $\gamma$ independently of other entries. The solution $\hat{\LLq}$ of \eqref{m:qtc} with $\lambda = \sum\limits_{j=1}^{k}\frac{\alpha_j^2}{\sqrt{\rho n_{j}^{(1)}}}$ is exact with a probability of at least $1-cn^{-10}$, provided that
			\begin{equation*}
		\rm rank(\Lq_{0_{(j)}}) \le \frac{\rho_rn_j^{(2)}}{\mu({\rm log}n_j^{(1)})^2} \quad and \quad\gamma \le \gamma_s,~ j=1,2,\cdots,k,
		\end{equation*}
		where, $c$, $\rho_r$ and $\gamma_s$ are positive numerical constants.
	\end{theorem}
	\begin{proof}
			Under the definition of quaternion tensor nuclear norm \eqref{snn}, model \eqref{m:qtc} reduces to the QMC model proposed in \cite{jia2018quaternion} when $\LLq$ and $\SSq$ reduce to quaternion matrix. Equivalently, model \eqref{m:qtc}  is equal to 
		\begin{equation} \label{m:qtc2}
			\begin{array}{rl}
				\underset{\LLq, \SSq}{\min}&
				\sum\limits_{j=1}^{k}\alpha_j\left\Vert  {\Lq_{(j)}} \right \Vert_{\ast} + \lambda \left\Vert  {\Sq_{(1)}} \right \Vert_{1}\\
				\ {\rm s.t. } &
				\mathcal{P}_{\Omega}\left({\Lq_{(j)}}+{\Sq_{(j)}}\right)={\Xq_{(j)}}, j=1,2,\cdots,k,
			\end{array}
		\end{equation}
		in which $\Lq_{(j)}$ and $\Sq_{(j)}$ are results of `unfold' function  \eqref{f:unfold} acted on $\LLq$ and $\SSq$  (we will use this notation in later models). 
		Model \eqref{m:qtc2} is a generalized QMC model by extending the first term of regularization function to the combination of $k$ nuclear norms.
		In other words, the model \eqref{m:qtc2} can be written as
		\begin{equation} \label{m:qtc3}
			\begin{array}{rl}
				\underset{\LLq, \SSq}{\min}&
				\sum\limits_{j=1}^{k}\alpha_j\left(\left\Vert  {\Lq_{(j)}} \right \Vert_{\ast} + \lambda_j \left\Vert  {\Sq_{(j)}} \right \Vert_{1}\right)\\
				\ {\rm s.t. } &
				\mathcal{P}_{\Omega}\left({\Lq_{(j)}}+{\Sq_{(j)}}\right)={\Xq_{(j)}}, j=1,2,\cdots,k.
			\end{array}
		\end{equation}
	 The parameter $\lambda$ in \eqref{m:qtc2} is denoted by $ \sum\limits_{j=1}^{k} \alpha_j\lambda_j$. Then \eqref{m:qtc3} is a convex combination of three QMC problems, thus the solution $\hat{\LLq}$ of \eqref{m:qtc2} is exact as long as they satisfy the exact recovery conditions respectively. According to Theorem 2 in \cite{jia2018quaternion}, we can get the conclusion.
	\end{proof}

	% Clearly, \eqref{m:qtc2} can be solved by the QMC algorithm after a minor modification, thus we consider the ADMM framework.
	Introducing two auxiliary  variables $\PPq$ and $\QQq$,  the augmented Lagrangian equation of problem \eqref{m:qtc2}  becomes
	\begin{equation} \label{lagrangian:qtc2}
		\begin{array}{rl}
			\underset{\LLq, \SSq,\PPq, \QQq}{\min}&
			\sum\limits_{j=1}^{k}\alpha_j\left\Vert  {\Lq_{(j)}} \right \Vert_{\ast}+\lambda \left\Vert  {\Sq_{(1)}} \right \Vert_{1}\\&+
			\sum\limits_{j=1}^{k}\frac{\beta_j}{2}\|\Lq_{(j)}-\Pq_{(j)}+\Yq_{(j)}/\beta_{j}\|_F^2\\
			&+ \frac{\mu}{2}\|\Sq_{(1)}-\Qqq_{(1)}+\Zq_{(1)}/\mu\|_F^2\\
			\ {\rm s.t. } &
			\mathcal{P}_{\Omega}\left({\PPq}+{\QQq}\right)={\XXq}.
		\end{array}
	\end{equation}
	where $\beta_j$ and $\mu$ are  penalty parameters, $\Yq_{(j)}$ and $\Zq_{(j)}$ are results of `unfold' function  \eqref{f:unfold} acted on two Lagrange multipliers $\YYq$ and $\ZZq$ that are two quaternion tensors. Now, we design an optimization algorithm to solve \eqref{lagrangian:qtc2} based on the ADMM framework.  Problem \eqref{lagrangian:qtc2} can be converted to two-block subproblems and each one contains two unknown variables:

	$[\SSq,~\PPq]$ subproblem: 
	\begin{equation*}
	\begin{aligned}
	&\underset{\SSq}{\min}~~\lambda \left\Vert  {\Sq_{(1)}} \right \Vert_{1}+ \frac{\mu}{2}\|\Sq_{(1)}-\Qqq_{(1)}+\Zq_{(1)}/\mu\|_F^2,\\
	&\underset{\PPq}{\min}~~\|\Lq_{(j)}-\Pq_{(j)}+\Yq_{(j)}/\beta_{j}\|_F^2, ~\!\!{\rm s. t.}~~\mathcal{P}_{\Omega}\left({\PPq}+{\QQq}\right)={\XXq}.
	\end{aligned}
\end{equation*}
	$[\LLq,~\QQq]$ subproblem:
	\begin{equation*} 
	\begin{aligned}
	&\underset{\LLq}{\min}~~\!\!\!\!
	\sum\limits_{j=1}^{k}(\alpha_j\left\Vert  {\Lq_{(j)}} \right\Vert_{\ast}+\frac{\beta_j}{2}\|\Lq_{(j)}-\Pq_{(j)}+\Yq_{(j)}/\beta_{j}\|_F^2),\\
    &\underset{\QQq}{\min}~~\!\!\!\!|\Sq_{(1)}-\Qqq_{(1)}+\Zq_{(1)}/\mu\|_F^2, ~~\!\!{\rm s. t.}~~\mathcal{P}_{\Omega}\left({\PPq}+{\QQq}\right)={\XXq}. 
    \end{aligned}
\end{equation*}

By these formulae,  the minimization problems of quaternion tensors are equivalently converted into the minimization problems of quaternion matrices. It seems that they can be feasibly solved  by the QMC iteration \eqref{a:admmqmc}.   However, one obstacle is the $\LLq$  subproblem that contains a convex combination of quaternion matrix norms. Fortunately,  we find that this problem has a closed-form solution. 
%This is essentially different to the QMC method.     
    \begin{theorem}
    The closed-form solution of $\LLq$ subproblem is $\frac{1}{k}\sum_{j=1}^{k}{\rm \texttt{fold}}_j\left(\texttt{approxQ}\left(\Pq_{(j)} -  (1/\beta_j)*\Yq_{(j)}, \alpha_j/\beta_j\right)\right)$.
    \end{theorem}
    \begin{proof}
    $\LLq$ subproblem is
    $$\underset{\Lq_{(j)}}{\min}~~
   \sum\limits_{j=1}^{k}(\alpha_j\left\Vert  {\Lq_{(j)}} \right\Vert_{\ast}+\frac{\beta_j}{2}\|\Lq_{(j)}-\Pq_{(j)}+\Yq_{(j)}/\beta_{j}\|_F^2)$$
   We find that the function $\sum\limits_{j=1}^{k}\alpha_j\left\Vert  {\Lq_{(j)}} \right\Vert_{\ast}$ is a sum of non-negative functions  $\alpha_j\left\Vert  {\Lq_{(j)}} \right\Vert_{\ast}$ that are independent with each other.  Then  $\LLq$ subproblem can be solved by finding  minimizers of $k$ subproblems $\alpha_j\left\Vert  {\Lq_{(j)}} \right\Vert_{\ast}+\frac{\beta_j}{2}\|\Lq_{(j)}-\Pq_{(j)}+\Yq_{(j)}/\beta_{j}\|_F^2$, respectively.  Suppose
   $\Lq_{(1)},\Lq_{(2)},\cdots,\Lq_{(j-1)},\Lq_{(j+1)},\cdots,\Lq_{(k)}$ have been known and  $\Lq_{(j)}$ is the only unknown variable. 
   The solution of   subproblem about $\Lq_{(j)}$ is
  \begin{equation*}
  	\begin{aligned}
%  	\underset{\Lq_{(j)}}{\min}F(\LLq)_j
        \Lq_{(j)}
  		&\!\!=\underset{\Lq_{(j)}}{\amin}~~
  		(\alpha_j\left\Vert  {\Lq_{(j)}} \right\Vert_{\ast}\!+\!\frac{\beta_j}{2}\|\Lq_{(j)}-\Pq_{(j)}+\Yq_{(j)}/\beta_{j}\|_F^2)\\
  		&\!\!=\underset{\Lq_{(j)}}{\amin}~~
  		(\frac{\alpha_j}{\beta_j}\left\Vert  {\Lq_{(j)}} \right\Vert_{\ast}\!+\!\frac{1}{2}\|\Lq_{(j)}-\Pq_{(j)}+\Yq_{(j)}/\beta_{j}\|_F^2)\\
  		&\!\!=\texttt{approxQ}\left(\Pq_{(j)} -  (1/\beta_j)*\Yq_{(j)}, \alpha_j/\beta_j\right).
  	\end{aligned}
   \end{equation*}
   Here, the quaternion singular thresholding operator is employed in the computation. The derivation is entirely independent of the selection of $j$, so the mentioned
   minimization can be carried out for any $\Lq_{(j)}, j = {1,\cdots,k}.$ 
%    	Thus, for $\LLq$ subproblem, 
%    	the quaternion singular value thresholding function \texttt{approxQ} \eqref{e:approxq} can be used independently to find the optimal solution.
%     can be obtained independently by the quaternion singular value thresholding operator $\texttt{approxQ}$ \eqref{e:approxq}.
%    	The $\LLq$ subproblem is is strictly convex since it is evident that quaternion tensor nuclear norm  is a convex combination of quaternion matrix nuclear norms, so the subproblem has a unique solution. 
    	%According to \cite{jia2022quaternion} and \cite{miao2020qtc},
    	
    	 Each solution  $\Lq_{(j)}$   is the optimal  approximation of the $j$th unfolding of $\LLq$. So the closed-form solution of $\LLq$ subproblem is 
    	 $$\LLq=\frac{1}{k}\sum_{j=1}^{k}\texttt{fold}_j\left(\texttt{approxQ}(\Pq_{(j)} -  (1/\beta_j)*\Yq_{(j)}, \alpha_j/\beta_j)\right).$$
    \end{proof} 
    
The other three subproblems can be solved similarly. 
 For instance,  the $\SSq$ subproblem can be solved  by the shrinkage of quaternion operator   $\texttt{shrinkQ}$ \eqref{e:shrinkQ}, and  in fact, it has a closed-form solution:  
 $$\SSq=\texttt{fold}_1\left(\texttt{shrinkQ}(\Qqq_{(1)}-  (1/\mu)*\Zq_{(1)},\lambda/\mu)\right).$$

To summarize above analysis, we present a new RQTC algorithm in Algorithm \ref{a:qtc_admm} and  prove its convergence in Theorem \ref{t:convergence}.
	{\linespread{1.1}
		\begin{algorithm}[h]
			\caption{RQTC Algorithm}
			\label{a:qtc_admm} 
			%			Given an observed quaternion matrix $\Xq=X_0+X_1\iq+X_2\jq+X_3\kq\in\Qmn$, this algorithm computes a low-rank quaternion matrix    $\Lq=L_0+L_1\iq+L_2\jq+L_3\kq$  and a sparse quaternion matrix $\Sq=S_0+S_1\iq+S_2\jq+S_3\kq$  satisfying  \eqref{model}.
			\begin{algorithmic}[1]
				\State \textbf{Input:}
				\State \indent  An observed quaternion tensor {$\XXq\in \mathbb{Q}^{n_1\times n_2\cdots\times n_k}$}  
				\State \indent Known and unknown entries sets: $\Omega$ and $\overline{\Omega}$; 
				\State \indent Initialize $\LLq = \XXq $, $\SSq=\PPq=\QQq=\YYq=\ZZq=0$, $\mu,\lambda,\beta_j>0,\sum_{j=1}^{k}\alpha_j=1,j=1,2,\cdots,k$.

				\State \textbf{Output:}
				\State \indent A low-rank quaternion tensor    $\LLq$.
				\State \indent  A sparse quaternion tensor $\SSq$.
				\State \textbf{Main loop:}  
				\While {not converge}
				\State \indent 
				Update $\SSq$ and $\PPq$:
				\For {j = 1:k}
				\State \indent
				$\Lq_{(j)}=\texttt{unfold}_j(\LLq)$; $\Pq_{(j)}=\texttt{unfold}_j(\PPq)$; d$\Yq_{(j)}=\texttt{unfold}_j(\YYq)$;
				\State \indent
				$\Pq_{(j)}(\Omega)=(\beta_j \Lq_{(j)}(\Omega)+\beta_j \Xq_{(j)}(\Omega)-\beta_j \Sq_{(j)}(\Omega)+\Yq_{(j)}(\Omega) \Zq_{(j)}(\Omega))/2/\beta_j$;  
				\State \indent 
				$\Pq_{(j)}(\overline{\Omega})=\Lq_{(j)}(\overline{\Omega})+\Yq_{(j)}(\overline{\Omega})/\beta_j$;
				\State \indent 
				$\PPq_j=\texttt{fold}_j(\Pq_{(j)})$;
				\EndFor
				\State \indent 
				$
				\PPq=\frac{1}{k}\sum_{j=1}^{k}\PPq_j;$
				\State \indent  
				$\SSq =\texttt{shrinkQ}(\QQq-  (1/\mu)*\ZZq, \lambda/\mu)$;
				\State \indent  
				Update $\LLq$ and $\QQq$:
				\For {j = 1:k}
				\State \indent 
				$\Lq_{(j)} = \texttt{approxQ}(\Pq_{(j)} -  (1/\beta_j)*\Yq_{(j)}, \alpha_j/\beta_j)$;
				\State \indent 
				$\LLq_j=\texttt{fold}_j(\Lq_{(j)})$;
				\EndFor 
				\State \indent
				$\LLq=\frac{1}{k}$$\sum_{j=1}^{k}\LLq_j;$
				\State \indent 
				$\QQq(\Omega)=\XXq(\Omega)-\PPq(\Omega)$; 
				$\QQq(\overline{\Omega})=\SSq(\overline{\Omega})+\ZZq(\overline{\Omega})/\mu$;
				\State \indent  
				Update $\YYq$ and $\ZZq$:
				\State \indent
				$\YYq=\YYq+\mu*(\LLq-\PPq)$;   
				$\ZZq = \ZZq + \mu*(\SSq-\QQq)$;
				
				\EndWhile
			\end{algorithmic}
	\end{algorithm}}
	\begin{theorem}\label{t:convergence}
     Algorithm \ref{a:qtc_admm} exactly convergences to the optimal solution $(\LLq^*,\SSq^*)$ of problem \eqref{m:qtc2}.
	%The sequence $\LLq_{(j)},\SSq_{(j)}$ generated by Algorithm \ref{a:qtc_admm} converges to an optimal solution $(\LLq^*,\SSq^*)$ of problem \eqref{m:qtc2}.  
	\end{theorem}
    \begin{proof}
%	Based on observations, we can convert the above four subproblems to ($[\LLq,\QQq]$ and $ [\SSq,\PPq]$) two-block subproblems and each subproblem can be  solved in parallel. 
	
	Since all of the matrices mentioned in \eqref{lagrangian:qtc2} are quaternion
	matrices, we reformulate them to real forms. Taking $\Lq_{(j)} = L_0 +
	L_1\iq + L_2\jq+ L_3\kq \in \mathbb{Q}^{n_j\times \prod \limits_{i\neq j}n_i }$ as an example, we represent it
	with a real vector defined by
	$L_{c_{(j)}} = [{\rm vec}(L_0); {\rm vec}(L_1); {\rm vec}(L_2); {\rm vec}(L_3)] \in \mathbb{R}^{4n_1n_2\cdots n_k}
	,$
	where ${\rm vec}(L_i)$ denotes an ($n_1n_2\cdots n_k$)-dimensional vector generated by stacking the columns of $L_i$. Thus, the quaternion model
	\eqref{lagrangian:qtc2}  is mathematically equivalent to 
		\begin{equation} \label{lagrangian:rqtc2}
			\begin{array}{rl}
				\underset{L_c, S_c,P_c, Q_c}{\min}&
				\sum\limits_{j=1}^{k}\alpha_j\left\Vert  {L_{c_{(j)}}} \right \Vert_{\ast}+\lambda \left\Vert  {S_{c_{(1)}}} \right \Vert_{1}\\&+
				\sum\limits_{j=1}^{k}\frac{\beta_j}{2}\|L_{c_{(j)}}-P_{c_{(j)}}+Y_{c_{(j)}}/\beta_{j}\|_F^2\\
				&+ \frac{\mu}{2}\|S_{c_{(1)}}-Q_{c_{(1)}}+Z_{c_{(1)}}/\mu\|_F^2\\
				 {\rm s.t. } &
				\mathcal{P}_{\Omega}\left({P_c}+{Q_c}\right)={X_c}.
			\end{array}
		\end{equation}
	Problem \eqref{lagrangian:rqtc2} is a minimization problem about real variables. For clarification, we define the object function by
	\begin{align*}
			&F\left(\left[
			\begin{array}{c}  
			P_c \\
			S_c	
			\end{array}
		\right]\!\!,\!\!\left[
		\begin{array}{c}  
			L_c \\
			Q_c	
		\end{array}
		\right]\!\!,\!\!\left[
		\begin{array}{c}  
			Y_c \\
			Z_c	
		\end{array}
		\right]		
		\right)\!:=
			\sum\limits_{j=1}^{k}\alpha_j\left\Vert  {L_{c_{(j)}}} \right \Vert_{\ast}+\lambda \left\Vert  {S_{c_{(1)}}} \right \Vert_{1} +\\
			&\sum\limits_{j=1}^{k}\!\frac{\beta_j}{2}\!\|L_{c_{(j)}}\!\!-\!\! P_{c_{(j)}}+Y_{c_{(j)}}/\beta_{j}\|_F^2
			+ \frac{\mu}{2}\|S_{c_{(1)}}\!-\!Q_{c_{(1)}}+Z_{c_{(1)}}/\mu\|_F^2.
\end{align*}
From  \cite[Proposition 2]{jia2018quaternion},   the convex
envelope of the function $\phi(\Xq) = \rm rank(\Xq)$ on $S := \{\Xq \in \mathbb{Q}^{n_1\times n_2}
	|\|\Xq\| \le 1\}$ can be expressed as
$$\phi_{envo}(\Xq) = \|\Xq\|_{\ast}.$$
	Therefore, the nuclear function $\|\cdot\|_{\ast}$ is convex and closed.
	On the other hand, the $\ell_1$ function  $\|\cdot\|_1$ is obviously  convex and closed.  Under this
	circumstance, the optimization problem \eqref{lagrangian:rqtc2} fits the framework  of
	ADMM. We can iteratively update all variables as follows:
\begin{align*}
	\left[
	\begin{array}{c}  
		P_c^{(t+1)} \\
		S_c^{(t+1)}
	\end{array}\!
    \right] &= \underset{P_c, S_c}{\amin}~\! F\left(
    \left[
    \begin{array}{c}  
    	L_c^{(t)} \\
    	Q_c^{(t)}
    \end{array}\!\!
\right]\!,\!\left[
\begin{array}{c}  
	P_c \\
	S_c
\end{array}\!
\right]\!,\!\left[
\begin{array}{c}  
	Y_c^{(t)} \\
	Z_c^{(t)}
\end{array}\!\!
\right]\!
    \right),\\
    \left[
    \begin{array}{c}  
    	L_c^{(t+1)} \\
    	Q_c^{(t+1)}
    \end{array}\!
    \right] &= \underset{L_c, Q_c}{\amin} ~\! F\left(
    \left[
    \begin{array}{c}  
    	L_c \\
    	Q_c
    \end{array}\!\!
    \right]\!,\!\left[
    \begin{array}{c}  
    	P_c^{(t+1)} \\
    	S_c^{(t+1)}
    \end{array}\!\!
    \right]\!,\!\left[
    \begin{array}{c}  
    	Y_c^{(t)} \\
    	Z_c^{(t)}
    \end{array}\!\!
    \right]\!
    \right),\\
    \left[
    \begin{array}{c}  
    	Y_c^{(t+1)} \\
    	Z_c^{(t+1)}
    \end{array}\!\!
    \right] &=  
    \left[
    \begin{array}{c}  
    	Y_c^{(t)}+\mu(L_c^{(t+1)}-P_c^{(t+1)}) \\
    	Z_c^{(t)}+\mu(S_c^{(t+1)}-Q_c^{(t+1)})
    \end{array}\!\!
    \right].
\end{align*}
	Here, we denote
    $\left[\begin{array}{c}  
    	P_c\\
    	S_c
    \end{array}\right],\left[\begin{array}{c}  
    L_c\\
    Q_c
\end{array}\right]$ and $\left[\begin{array}{c}  
Y_c\\
Z_c
\end{array}\right]$
	by
	$x$,$y$ and $z$, respectively. Then the above equation is consistent with the following
	equations in \cite{boyd2011distributed}: 
	\begin{align*}
	x^{(t+1)} &= \underset{x}{\amin}~ F(x,z^{(t)},y^{(t)}),\\
	z^{(t+1)} &= \underset{x}{\amin}~ F(x^{(t+1)},z,y^{(t)}),\\
	y^{(t+1)} &= y^{(t)} +\mu(Ax^{(t+1)}+Bz^{(t+1)}-c).
	\end{align*}
As a result, this falls essentially in the two-block ADMM framework and the convergence is theoretically guaranteed according to \cite{boyd2011distributed}.
    \end{proof}
    
    \begin{remark} It is worth mentioning that this model performs better than the real tensor completion model, 
	because the intrinsic color structures are totally retained during the computation process for the quaternion tensor, while unfold process may completely obliterate the three channels of color pixel.
	\end{remark}

	\begin{remark}  A color image can be seen as a color video with only one frame and its representation is a quaternion matrix. That is, if $n_3=1$ then a 3-mode quaternion tensor $\XXq$ reduces to a quaternion  matrix $\Xq$. So the proposed RQTC method is surely a generalization of  the QMC method \cite{jia2018quaternion}.
The incoherence conditions and the assumption of low-rank and sparsity for robust quaternion tensor recovery problem surely cover those in \cite[Theorem 2]{jia2018quaternion} for robust quaternion matrix quaternion recovery problems. 
	\end{remark}
	
	\begin{remark}
	In model \eqref{m:qtc}, the  definitions of quaternion tensor nuclear norm and  $\ell_1$ norm in \eqref{snn} are
	inspired by \cite{liu2009}, in which  the SNN is established as the nuclear norm for real tensors. 
	\end{remark}

	In the above, we have presented a novel RQTC model with the  exact recovery theorem and a new ADMM-based algorithm with a convergence proof.  They are feasible and efficient to restore quaternion tensors from partial and/or corrupted entries under the condition that the assumption of Theorem \ref{t:exactRQTC} is satisfied.  However,  the low-rank condition sometimes does not hold in practical applications. For instance, quaternion tensor that represents  color video is  of high-rank when  color video contains high frequency information. So we need to improve our model and algorithm further.

%	
%	However, Algorithm \ref{a:qtc_admm} involving matrices are of large scale and their low-rank or sparsity properties are not surely guaranteed. RQTC can restore quaternion tensor under partial or corrupted observations, but the results of the natural color videos recovery are not very good. Consequently, we must establish a new learning patch-based approach to equivalently some small scales and learn low-rank properties adeptly method.

	%	Different from recovering unknown or nosiy color videos, extracting the foreground in videos need make sure that the background images are linearly correlated with each other. This means each color video with color frame size $n_1 \times n_2$ and frame number $m$ reshapes to the $n_1n_2 \times m$ matrix. Thus, in the process of solving model (\ref{model_video}), the tensor $\mathcal{L}$ should convert to matrix $\bf L$.  Essentially, extracting the foreground from color videos is also a quaternion matrix completion problem.
	\subsection{LRL-RQTC model}
	Now we present an improved RQTC model by introducing a low-rank learning method.  
	For the convenience of description, we concentrate on  $3$-mode quaternion tensor  $\XXq \in \mathbb{Q}^{n_1\times n_2\times n_3}$ that represents color video and  use the engineer language instead of the mathematician language.

Since color video is often of large-scale,  we set a  window for searching low-rank information and denote the part of quaternion tensor in this window by adding a subscript $t$. That is, $\XXq_t$ denotes a small quaternion tensor of $\XXq$ in a fixed searching window.  
%
%	When the scale of each slice of color video is huge, it is practical to set a window for searching similar patches. 
	%Setting a window for similar patches searching is  feasible when.

	 Suppose we set $p$ windows totally, in other words, we divide the large tensor into $p$ smaller ones. 
	 A $3$-mode quaternion tensor  is a stack of horizontal, lateral and frontal slices. We choose a  series of overlapping patches of $\XXq_{t}$ (the $t$th searching window of tensor $\XXq$) from three kinds of slices,  respectively, and classify them into $\ell_j$ classes $( j=1,2,3 )$. Then we vectorize each patch of the $s$th class ($1 \le s \le \ell_j $) and rearrange them into a quaternion matrix, denoted by $\mathscr{F}_s^j(\XXq_{t})$, 
%	 $\mathscr{L}_s(\XXq_{t})$,  $\mathscr{F}_s(\XXq_{t})$. 
	 In other words,  $\mathscr{F}_s^j(\XXq_{t})$
%	 $\mathscr{L}_s(\XXq_{t})$,  $\mathscr{F}_s(\XXq_{t})$ 
	 is a low-rank quaternion matrix  generated by the $s$th class of similar patches from the $j$th type of  slice. 
	Define the classification function of similar patches by 
	\begin{align}\label{e:h}
%	\mathcal{H}(\XXq_{t})&=[\mathscr{H}_1(\XXq_{t}),~\mathscr{H}_2(\XXq_{t}),\cdots,\mathscr{H}_{\ell}(\XXq_{t})],\\ \label{e:l}
%	\mathcal{L}(\XXq_{t})&=[\mathscr{L}_1(\XXq_{t}),~\mathscr{L}_2(\XXq_{t}),\cdots, \mathscr{L}_{\ell}(\XXq_{t})],\\ \label{e:f}
	\mathcal{F}_j(\XXq_{t})&=[\mathscr{F}_1^j(\XXq_{t}),~\mathscr{F}_2^j(\XXq_{t}),\cdots, \mathscr{F}_{\ell}^j(\XXq_{t})].
	\end{align}
	%where $j$ denotes the number of searching window.
	The function is invertible and the inverse  is defined by
	\begin{align}\label{e:f-1}
%	\mathcal{H}^{-1}\left([\mathscr{H}_1(\XXq_{t}),~\mathscr{H}_2(\XXq_{t}),\cdots,\mathscr{H}_{\ell}(\XXq_{t})]\right)&=\XXq_{t},\\
%	\mathcal{L}^{-1}\left([\mathscr{L}_1(\XXq_{t}),~\mathscr{L}_2(\XXq_{t}),\cdots, \mathscr{L}_{\ell}(\XXq_{t})]\right)&=\XXq_{t},\\
	\mathcal{F}_j^{-1}\left([\mathscr{F}_1^j(\XXq_{t}),~\mathscr{F}_2^j(\XXq_{t}),\cdots, \mathscr{F}_{\ell}^j(\XXq_{t})]\right)&=\XXq_{t}.
    \end{align}
	%
	%	Now we  propose a low rank learning framework for color image process or foreground detection from color videos completion under quaternion representation.
	%
	
	Now we introduce a learning strategy into RQTC model and present a new low-rank learning robust quaternion tensor completion model (LRL-RQTC): 
%	\begin{flalign} \label{m:LRL-QTC}
%		\begin{array}{rl}
%			\underset{\LLq,\SSq}{\min}	&\sum\limits_{t=1}^{p}\sum\limits_{i=1}^{\ell}
%			(\alpha_1\left\Vert  {\mathtt{H}_{i}\left(\LLq_t\right)} \right \Vert_{\ast}+\lambda_i\|\mathtt{H}_{i}(\SSq_t)\|_1 + 	
%			\alpha_2\left\Vert  {\mathtt{L}_{i}\left(\LLq_t\right)} \right \Vert_{\ast} \\[5mm]
%			&+ \lambda_i\|\mathtt{L}_{i}(\SSq_t)\|_1 +
%			\alpha_3\left\Vert  {\mathtt{F}_{i}\left(\LLq_t\right)} \right \Vert_{\ast}+\lambda_i\|\mathtt{F}_{i}(\SSq_t)\|_1)
%		\end{array}
%	\end{flalign}
\begin{equation} \label{m:LRL-QTC}
	\begin{array}{rl}
		\underset{\LLq,\SSq}{\min}&
		\sum\limits_{j=1}^{3}\sum\limits_{t=1}^{p}\sum\limits_{s=1}^{\ell_j}
		(\alpha_j\left\Vert  {\mathscr{F}_{s}^j\left(\LLq_t\right)} \right \Vert_{\ast}+\lambda_s\|\mathscr{F}_{s}^j(\SSq_t)\|_1)\\
		{\rm s.t. } &
		\mathcal{P}_{\Omega}\left({\LLq}+{\SSq}\right)={\XXq}.
	\end{array}
%	\begin{aligned}
%		&\underset{\LLq,\SSq}{\min}\quad	\sum\limits_{j=1}^{3}\sum\limits_{t=1}^{p}\sum\limits_{s=1}^{\ell_j}
%		(\alpha_j\left\Vert  {\mathscr{F}_{s}^j\left(\LLq_t\right)} \right \Vert_{\ast}+\lambda_s\|\mathscr{F}_{s}^j(\SSq_t)\|_1 )\nonumber \\
%		&{\rm s.t. } \quad \mathcal{P}_{\Omega}\left({\LLq}+{\SSq}\right)={\XXq}
%	\end{aligned}
\end{equation}
%\begin{flalign*} \label{m:LRL-QTC}
%	\underset{\LLq,\SSq}{\min}	&\sum\limits_{t=1}^{p}\sum\limits_{s=1}^{\ell_j}
%	(\sum\limits_{j=1}^{3}\alpha_j\left\Vert  {\mathscr{H}_{s}\left(\LLq_t\right)} \right \Vert_{\ast}+\lambda_s\|\mathscr{H}_{s}(\SSq_t)\|_1 + 	
%	\alpha_2\left\Vert  {\mathscr{L}_{s}\left(\LLq_t\right)} \right \Vert_{\ast} \\[3mm]
%	&+ \lambda_s\|\mathscr{L}_{s}(\SSq_t)\|_1 +
%	\alpha_3\left\Vert  {\mathscr{F}_{s}\left(\LLq_t\right)} \right \Vert_{\ast}+\lambda_s\|\mathscr{F}_{s}(\SSq_t)\|_1)&&
%\end{flalign*}
%	\begin{equation} \label{m:LRL-QTC}
%		\begin{array}{rl}
%			\underset{\LLq,\SSq}{\min}	\left(\alpha_1\sum\limits_{t=1}^{p}\sum\limits_{i=1}^{\ell_1}\left\Vert  {\mathcal{H}_{i}\left(\LLq_t\right)} \right \Vert_{\ast}+\lambda_i\|\mathcal{H}_{i}(\SSq_t)\|_1\right) + \\	\left(\alpha_2\sum\limits_{t=1}^{p}\sum\limits_{i=1}^{\ell_2}\left\Vert  {\mathcal{L}_{i}\left(\LLq_t\right)} \right \Vert_{\ast}+\lambda_i\|\mathcal{L}_{i}(\SSq_t)\|_1\right) +\\
%			\left(\alpha_3\sum\limits_{t=1}^{p}\sum\limits_{i=1}^{\ell_3}\left\Vert  {\mathcal{F}_{i}\left(\LLq_t\right)} \right \Vert_{\ast}+\lambda_i\|\mathcal{F}_{i}(\SSq_t)\|_1\right)\\
%		\end{array}
%	\end{equation}
%\begin{flalign}
%	 {\rm s.t. } \quad \mathcal{P}_{\Omega}\left({\LLq}+{\SSq}\right)={\XXq}&&
%\end{flalign}
%    $$ {\rm s.t. } \mathcal{P}_{\Omega}\left({\LLq}+{\SSq}\right)={\XXq}$$
    where $\mathscr{F}_{s}^j$ is a mapping transformation \eqref{e:h}.
	% $\mathcal{F}_i\left(\cdot\right)=F_{i}\circ\phi(\cdot): \mathbb{Q}^{  n_1 \times n_2 \times m} \rightarrow \Qq^{wh \times d}$. It contains two steps.
	Different from the prior NSS-QMC model that uses the distance function to search similar patches, here we introduce a new 2DQPCA-based classification function and propose a new LRL-RQTC model with adaptively low-rank learning. The model \eqref{m:LRL-QTC} is a minimization issue consisting of three subproblems and independent of each other. For the convenience of the narrative, we only present the operation on the frontal slice in the following part, i.e.
	\begin{equation} \label{m:LRL-QTC-f}
		\begin{array}{rl}
			\underset{\LLq,\SSq}{\min}&	
			\sum\limits_{t=1}^{p}\sum\limits_{s=1}^{\ell}\left\Vert  {\mathscr{F}_{s}\left(\LLq_t\right)} \right \Vert_{\ast}+\lambda_s\|\mathscr{F}_{s}(\SSq_t)\|_1\\
				{\rm s.t. } &
			\mathcal{P}_{\Omega}\left({\LLq}+{\SSq}\right)={\XXq}.
		\end{array}
	\end{equation}
	
	\begin{figure*}
	\begin{minipage}{1.2\linewidth}
			\includegraphics[height=8cm, width=18cm]{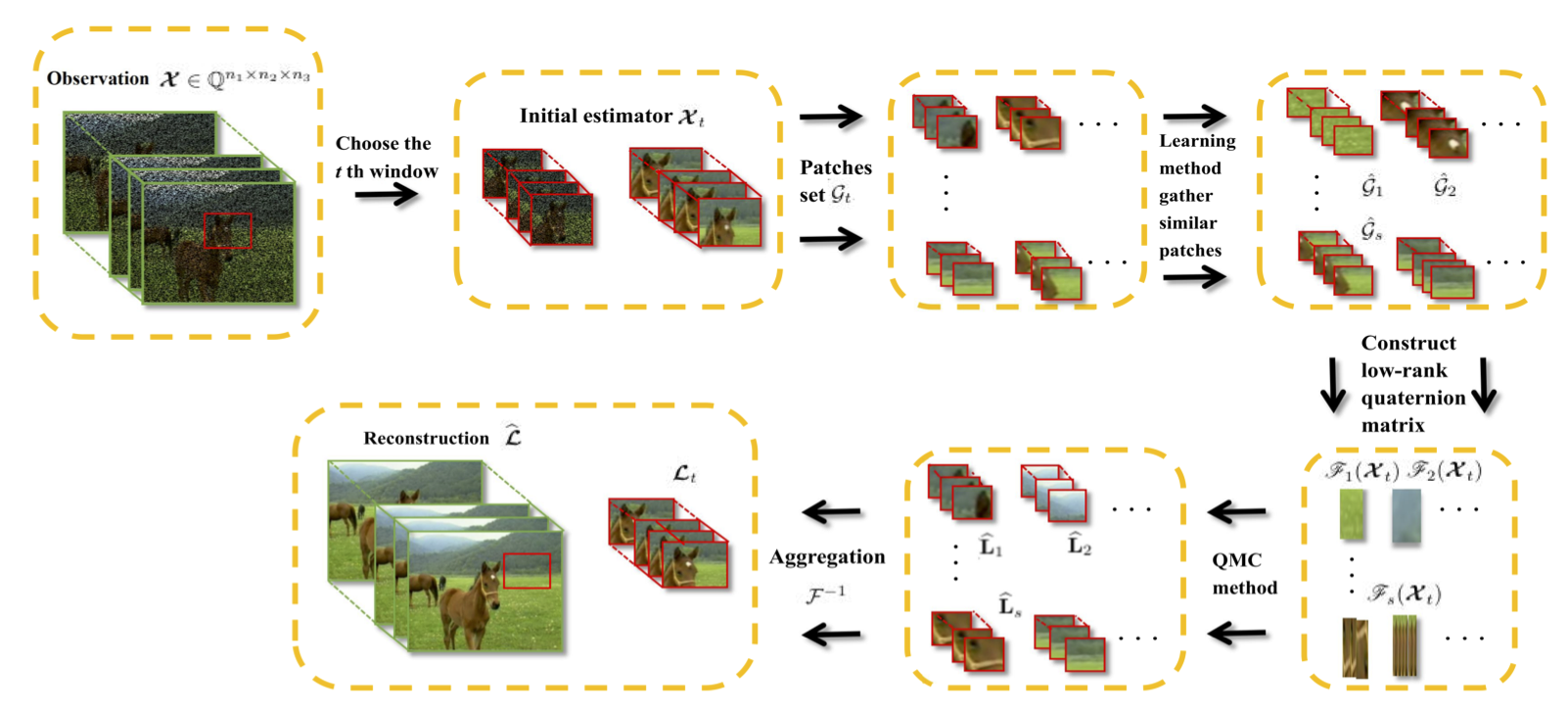}
	\end{minipage}
	\caption{Flowchart of the patched-based low-rank  learning method for RQTC problem.}\label{f:flowchart}
	\end{figure*}
	The flowchart of the patched-based learning method is shown
	in Fig. \ref{f:flowchart}. Firstly, choose the $t$th window $(1\le t\le p)$ and get $n$ overlapping patches $\Yq_{t}^{(i,j)}$ with size  $w\times h$  ($0< w \le n_1$, $0< h\le n_2$), covering each frontal slice of $\XXq_{t}$. Then we gather them into a set
	$\mathcal{G}_{t}$ 
	%	:= \{\Yq_{i,j}\in \mathbb{Q}^{ w \times h}\}$,
	\begin{equation}\label{e:Gpatches} \mathcal{G}_{t}:=\{\Yq_{t}^{(i,j)}\in \mathbb{Q}^{ w \times h}\},
	\end{equation}
	where $(i,j)$ denotes the location of a patch.
%	 We assume there are totally $n$ patches in group $\mathcal{G}_{(t)}$ and
%	fix $w\times h $ scale of one patch, $0<w\le n_1$ and $0<h\le n_2$. 
	%	
%	$	\widehat{\Lq}_{s}$
	Secondly, we set $\ell$ exemplar  patches which are non-overlapping. After choosing them, we calculate their eigen subspace $\Vq$. Then we find a number of  patches   most similar to the exemplar patches by  classification in group $\mathcal{G}_{t}$. The entire process is shown in Algorithm \ref{a:2dqpca}.
	For the $s$th exemplar patch,  similar patches are stored in $	\hat{\mathcal{G}}_{s} (s=1,\cdots,\ell)$, which is the subset of $\mathcal{G}_{t}$.
	%	\begin{equation}\label{e:grouppatch}
	%		\mathcal{G}_{i,j} =\left\{ \Yq_{s,t}:
	%		\min \|\Yq_{s,t}\Vq- \Yq_{i,j}\Vq\|\right\}\!,
	%	\end{equation}
	By the 2DQPCA technology, we successfully achieve better performance in matching similar patches by learning low dimensional representation and forming small scale and quantity matrices to reduce CPU time.  According to the result of classification from 2DQPCA, the number of similar patches in $\hat{\mathcal{G}}_{s}$ is not fixed,  denoted by $|\hat{\mathcal{G}}_{s}|=d_s$ with $d_s$ being a  positive integer. 
%	In the practical experiment, in order to reduce the size of the calculated matrix and compute time, we can fix  $d_s = d$ which $d$ is a certain small constant after preprocessing data.
	Finally,  we  stack the quaternion matrix from $\hat{\mathcal{G}}_{s}$ to the quaternion column vector and put them together lexicographically to construct a new quaternion matrix 
	\begin{equation}\label{e:fsx}
		\mathscr{F}_s(\XXq_t) = [{\tt vec}(\Yq_{t}^{(i_1,j_1)}),\cdots,{\tt vec}(\Yq_{t}^{(i_{d_s},j_{d_s})})] \in \mathbb{Q}^{(wh)\times d_s},
	\end{equation}
	where $\Yq_{t}^{(i_{k},j_{k})}$ denotes  the $k$-th element of $\hat{\mathcal{G}}_{s}$. 
	Thus, this 2DQPCA-based classification process learns $\ell$ small low-rank quaternion matrices stored in the set $\mathcal{F}(\XXq_t)$.
	
	Then, we repeat the above learning process on each window of horizontal, lateral and frontal slices of tensor $\XXq$ until the low-rank conditions are satisfied.

	{\linespread{1.1}
		\begin{algorithm}[h]
			\caption{2DQPCA-based classification function}
			\label{a:2dqpca} 
			
			\begin{algorithmic}[1]
				\State \textbf{Input:}
				\State \indent  The set $\mathcal{G}_{t}$ in \eqref{e:Gpatches}
				\State \textbf{Output:}
				\State \indent Low-rank set $\mathscr{F}_s(\XXq_t)$ in  \eqref{e:fsx}.
				
				\State \textbf{Main loop:}  
				\State \indent Compute the covariance matrix of $\ell$ exemplar patches from $\mathcal{G}_{t}$: $\Cq = \frac{1}{\ell-1}\sum\limits_{s=1}^{\ell}{\Phi_s^{\ast}}{\Phi_s}\ \in \mathbb{Q}^{h \times h}$
				where $\Phi_s=(\Yq_{t}^{(i_{s},j_{s})}-\Psi)$,
				$\Psi=\frac{1}{\ell}\sum\limits_{s=1}^{\ell}\Yq_{t}^{(i_{s},j_{s})}$, $s=1,\cdots,\ell$.
				\State \indent   Compute the eigenvalues of $\Cq$ and their eigenvectors, denote by $(\lambda_1,\vq_1),\cdots,(\lambda_h,\vq_h)$. Define projection subspace as $\Vq = \rm span \{\vq_1,\cdots,\vq_h\}$.
%				 Now we use the optimal projection vectors $(\qq_1,\cdots,\qq_h)$ for feature extraction.
				\State \indent  
				Compute the projections of $\ell$ training patches, $$\Pq_s=\Phi_s\Vq \in \mathbb{Q}^{h \times h},\quad s=1,\cdots,\ell.$$
				\State \indent 
				For the rest samples in  $\mathcal{G}_{t}$, compute their feature matrix, $$\widehat{\Pq}_k=(\Yq_{t}^{(i_{k},j_{k})}-\Psi)\Vq,\quad k=1,  \cdots,n-\ell$$
				\State \indent 
				Solve the optimization problems
				$$y(k)={\rm arg}\min\limits_{1\le k\le n-\ell} \|\widehat{\Pq}_k- \Pq_s\|$$
				for $k=1,\cdots,n-\ell$.
				\State \indent 
				Find the same identity in $y$ and gather them in $\hat{\mathcal{G}}_{s}$ lexicographically.
				\State \indent
				Vectorized each element in $\hat{\mathcal{G}}_{s}$ to form $\mathscr{F}_s(\XXq_t)$ defined by \eqref{e:fsx}.
			\end{algorithmic}
	\end{algorithm}}

	Now we prove that the 2DQPCA-based classification function (Algorithm \ref{a:2dqpca}) generates a low $\delta$-rank matrix.  We refer to the definition of low $\delta$-rank \cite{jia2022quaternion}. 
	\begin{definition}\cite{jia2022quaternion}\label{d:lowrank}
		 A quaternion matrix $\Aq$ is called of $\delta$-rank $r$ if it has $r$ singular values bigger than $\delta > 0$.
	\end{definition}
We can see that the matrix is of low-rank when $r$ tends to zero. Based on the above definition, we give the following theorem to prove that $\mathscr{F}_s(\XXq_t)$  is a low $\delta$-rank matrix.
	\begin{theorem}\label{cor:deltarank}
		Suppose that each $\mathscr{F}_s(\XXq_t) \in \mathbb{Q}^{  (wh) \times d_s }$ generated by Algorithm \ref{a:2dqpca}  satisfies
		$\|\widehat{\Pq}_k- \Pq_s\|_F\le \frac{\sqrt{2}}{2}\delta $ and
		has  the singular value decomposition:
		$\mathscr{F}_s(\XXq_t)=\Uq\Sigma\Vq^*$,
		where
		$\Sigma=\texttt{diag}(\sigma_1,\cdots,\sigma_{d_s}),~\sigma_1\ge\cdots\ge\sigma_{d_s}\ge 0,\Vq =[\vq_1,\cdots,\vq_{d_s}]$.
		Let $r~(\le {d_s})$ be the least positive integer such that
		\begin{equation}\label{e:rij}
			\sum\limits_{k=1}^{{r}-1}(\sigma_{k}^2-\sigma_{r}^2)|\wq_{k}|^2\ge \sum\limits_{k=r+1}^d (\sigma_{r}^2-\sigma_k^2)|\wq_k|^2,
		\end{equation}
		where $[\wq_1,\cdots,\wq_d]^T=\vq_i-\vq_j$.
		Then $\sigma_{r}\le \delta$ and thus the $\delta$-rank of $\mathscr{F}_s(\XXq_t)$ is less than  $r$.
	\end{theorem}
	\begin{proof} Refer to \cite[ Theorem 3.1]{jia2022quaternion}, it is sufficient to demonstrate that $\|\xq_i-\xq_j\|_2\le \sqrt{2} \delta$, where $\xq_i,\xq_j$ are any two columns of $\mathscr{F}_s(\XXq_t)$.
		  
		  From \eqref{e:fsx}, we choose two columns of $\mathscr{F}_s(\XXq_t)$: ${\tt vec}(\Yq_{t}^{(i_1,j_1)})$ and ${\tt vec}(\Yq_{t}^{(i_2,j_2)})$. Then,
		$\|\xq_i-\xq_j\|_2=\|{\tt vec}(\Yq_{t}^{(i_1,j_1)})-{\tt vec}(\Yq_{t}^{(i_2,j_2)})\|_2 = \|\Yq_{t}^{(i_1,j_1)}-\Yq_{t}^{(i_2,j_2)}\|_F = \|\Yq_{t}^{(i_1,j_1)}-\Yq_{t}^{(i_s,j_s)}+\Yq_{t}^{(i_s,j_s)}-\Yq_{t}^{(i_2,j_2)}\|_F = \|(\Yq_{t}^{(i_1,j_1)}-\Psi)\Vq-(\Yq_{t}^{(i_s,j_s)}-\Psi)\Vq+(\Yq_{t}^{(i_s,j_s)}-\Psi)\Vq-(\Yq_{t}^{(i_2,j_2)}-\Psi)\Vq\|_F =  \|\widehat{\Pq}_{1}-\Pq_{s}+\Pq_{s}-\widehat{\Pq}_{2}\|_F \le \|\widehat{\Pq}_{1}-\Pq_{s}\|_F + \|\Pq_{s}-\widehat{\Pq}_{2}\|_F \le \sqrt{2}\delta $.
		%		Since each element of $\mathcal{G}_{i,j}$ satisfy $\|\Yq_{s,t}-\Yq_{i,j}\|_F\le \frac{\sqrt{2}}{2}\delta$,
		%		$\|\Yq_{s_1,t_1}-\Yq_{s_2,t_2}\|_F
		%		=\|\Yq_{s_1,t_1}-\Yq_{i,j}+\Yq_{i,j}-\Yq_{s_2,t_2}\|_F\le \|\Yq_{s_1,t_1}-\Yq_{i,j}\|_F+\|\Yq_{i,j}-\Yq_{s_2,t_2}\|_F\le \sqrt{2}\delta$
		%		holds for any two columns ${\tt vec}(\Yq_{s_1,t_1})$ and ${\tt vec}(\Yq_{s_2,t_2})$ of $\Xq_{i,j}$.
	\end{proof}
	%	%%%%
	\begin{remark}
	Theorem \ref{cor:deltarank} describes the matrix generated by 2DQPCA-based classification function has the $\delta$-rank less than r. Actually, according to the definition of approximately low-rank matrix \cite{errorchen2022}:
	$$\sum_{i=1}^{d_i}|\sigma_i|^q\le \rho, 0 \le q <1,
	$$
	where $\rho$ is a numerical value, $\sigma_i$ is the singular values of matrix and the matrix reduces to low-rank when $q=0$. 
	
    If we get singular values $\sigma_i$ of $\mathscr{F}_s(\XXq_t)$, after a simple calculation, 
    $\rho \ge \sum_{i=1}^{r}|\sigma_i|^q \ge \delta^qr$, then $r\le \rho \delta^{-q}$. 
%    For any $\delta>0$, the $\delta$-rank of $\Xq_{(t_i)}$ is a small constant.
	\end{remark}

	%

	%\begin{algorithm}\label{a:qmc_admm}{\bf Algorithm \eqref{a:qmc_admm} (QMC Algorithm).}
	%Given an observed quaternion matrix $\Xq=X_0+X_1\iq+X_2\jq+X_3\kq\in\Qmn$, this algorithm computes a low-rank quaternion matrix    $\Lq=L_0+L_1\iq+L_2\jq+L_3\kq$  and a sparse quaternion matrix $\Sq=S_0+S_1\iq+S_2\jq+S_3\kq$  satisfying  \eqref{e:qmcpatchmodel}.
	%%
	%\begin{enumerate}
	%\item[$1.$] Input: $X=[X_0~X_1~X_2~X_3]$;  the set of the indexes of known pixels $\Omega$; the set of the indexes of unknown pixels $\overline{\Omega}$; $L = X$; $S=P=Q=Y=Z=\texttt{zeros}(n_1, 4*n_2)$; $\mu,\lambda>0$.
	%\item[$2.$] While not converge
	%\item[$3.$] \quad  Update $L$ and  $Q$:
	% \item[]\qquad   $L = \texttt{approxQ}(P -  (1/\mu)*Y, 1/\mu)$;
	% \item[]\qquad   $Q(\Omega)=X(\Omega)-P(\Omega)$; \quad
	%             $Q(\overline{\Omega})=S(\overline{\Omega})+Z(\overline{\Omega})/\mu$;
	%\item[$4.$] \quad  Update $S$ and $P$:
	%\item[] \qquad $S =\texttt{shinkQ}(Q-  (1/\mu)*Z, \lambda/\mu)$;
	%\item []       \qquad  $P(\Omega)=(\mu L(\Omega)+\mu X(\Omega)-\mu S(\Omega)+Y(\Omega)-$
	%\item[]         \qquad $Z(\Omega))/2/\mu$;  \quad $P(\overline{\Omega})=L(\overline{\Omega})+Y(\overline{\Omega})/\mu$;
	%\item[$5.$]  \quad  Update $Y$ and $Z$:
	%\item[]         \qquad  $Y=Y+\mu*(L-P)$;   $Z = Z + \mu*(S-Q)$;
	%\item[$6.$] end
	%\end{enumerate}
	%\end{algorithm}

%	\subsection{Low Rank Learning QTC Algorithm}
	Next, we propose a new LRL-RQTC algorithm based on learning scheme for solving \eqref{m:LRL-QTC}. Based on ADMM framework, by introducing two quaternion tensors  $\PPq$ and $\QQq$, we minimize the following equivalent problem:
	\begin{equation} \label{m:qtcqmc}
		\begin{array}{rl}
			\underset{\LLq,\SSq,\PPq,\QQq,}{\min}&
			\sum\limits_{t=1}^{p}\sum\limits_{s=1}^{\ell}\left\Vert  {\mathscr{F}_{s}\left(\LLq_t\right)} \right \Vert_{\ast}+\lambda_s\|\mathscr{F}_{s}(\SSq_t)\|_1\\
			{\rm s.t. } &
			\mathcal{P}_{\Omega}\left({\LLq}+{\SSq}\right)={\XXq},
			\PPq=\LLq,\QQq=\SSq.
		\end{array}
	\end{equation}
	%The augmented Lagrangian function is given as follows by
	%attaching multipliers $\Yq$ and $\Zq$:
	%	\begin{equation} \label{ucmodel}
	%		\begin{array}{rl}
	%			\underset{\mathcal{L},\mathcal{S},\Fq_{(p)},\Tq_{(1)}}{\min}&
	%			\sum\limits_{i=1}^{\ell}\left(\sum\limits_{p=1}^{3}\alpha_p\left\Vert {\mathcal{F}_{i}(\Lq_{(p)})} \right \Vert_{\ast} + \frac{\mu}{2} \left\Vert {\Lq_{(p)}-\Fq_{(p)}+\Yq_{(p)}/\mu}\right\Vert _{F}^2\right)\\
	%			+\lambda_i\|\mathcal{F}_{i}(\Sq_{(1)})\|_1\\+\frac{\mu}{2} \left\Vert {\bf S-\bf Q+\bf Z/\mu}\right\Vert _{F}^2\\
	%			\ {\rm s.t.} &
	%			\mathcal{P}_{\Omega}\left({\bf P}+{\bf Q}\right)={\bf X},
	%		\end{array}
	%	\end{equation}
	Actually, the solving method of the problem \eqref{m:qtcqmc} converts to QMC method. Then we apply the QMC algorithm \eqref{a:admmqmc} to  solve \eqref{m:qtcqmc}, which can be broken down into $p\ell$ independent subprocesses and thus, they can be performed in parallel.  Under the assumption of low-rank condition, we can get the low-rank  reconstruction  $\mathscr{F}_{s}\left(\hat{\LLq}_t\right)$ and the sparse composition $\mathscr{F}_{s}\left(\hat{\SSq}_t\right)$.  Then we can get a good approximation of a quaternion tensor. 
	%	By the averaging technique, a good approximation, denoted by $\widehat{\Lq}$,  to the original color image  is reconstructed  from $m$ reconstructed subgroups   $\widehat{\mathcal{G}}_{i,j}$'s.
	
	Based on the above analysis, we summarize the proposed LRL-RQTC algorithm and present the  pseudo-code in  Algorithm \ref{a:pqmc}.

	{\linespread{1.1}
		\begin{algorithm}[h]
			\caption{LRL-RQTC  Algorithm}
			\label{a:pqmc} 
			%			Given an observed color image $\Xq=X_0+X_1\iq+X_2\jq+X_3\kq\in\Qmn$ with $X_0\equiv 0$, this algorithm computes a   reconstruction    $\widehat{\Lq}=\widehat{L}_0+\widehat{L}_1i+\widehat{L}_2j+\widehat{L}_3k$ with $\widehat{L}_0\equiv 0$ and $\widehat{L}_1,~\widehat{L}_2,~\widehat{L}_3$ denoting the red, green, blue color channels, respectively.
			\begin{algorithmic}[1]
				\State \textbf{Input:}
				\State \indent  The observed quaternion tensor $\XXq \in\mathbb{Q}^{n_1\times   n_2\times n_3}$;
				\State \indent Known and unknown entries sets: $\Omega$ and $\overline{\Omega}$;     
				\State \indent Fix the size  $w\times h$ of patch; 
				\State \indent Determine $p$ windows  and $\ell$ exemplar patches ;
				\State \indent
				A tolerance  $\delta>0$;  constants $\lambda_i> 0$; weights $\alpha_j$ ;
%				\State \indent the dimension of the patching group $d$.
				\State \textbf{Output:}
				\State \indent A   reconstruction    $\widehat{\LLq}$;
				\State \indent A sparse component  $\widehat{\SSq}$.
				\State \textbf{Main loop:}  
				\For {t = 1:p}
%				\While {not converge}
				\State \indent Create the set $\mathcal{G}_{t}$ as in \eqref{e:Gpatches}. \State \indent Choose $\ell$  exemplar patches  $\Yq_{t}^{(i_s,j_s)}, s=1,\cdots,\ell$ with $(i_s,j_s)\in\Omega$.
				\For {s = 1:$\ell$}
				\State \indent  Apply the 2DQPCA-based classification function (Algorithm \ref{a:2dqpca}) to learn some similar patches of $\Yq_{t}^{(i_s,j_s)}$ and generate them as a  low-rank quaternion matrix   $\mathscr{F}_s(\XXq_{t})$ as in \eqref{e:fsx}.
				\State \indent  
				Find a  low-rank approximation $\mathscr{F}_{s}\left(\hat{\LLq}_t\right)$  and a sparse component $\mathscr{F}_{s}\left(\hat{\SSq}_t\right) $ of $\mathscr{F}_s(\XXq_{t})$ by the QMC  iteration \eqref{a:admmqmc} .
				%				\State \indent 
				%				Compute the reconstruction $\widehat{\Lq}$  and $\widehat{\Sq}$ from $\ell$ reconstructed better approximate matrices   $\widehat{\Lq}_{i,j} $ and $\widehat{\Sq}_{i,j}$.
				\EndFor
%				\EndWhile
				\State \indent 
				Use $\mathcal{F}^{-1}$ function on  $[\mathscr{F}_{1}\left(\hat{\LLq}_t\right),\cdots, \mathscr{F}_{\ell}\left(\hat{\LLq}_t\right)]$  and
				$[\mathscr{F}_{1}\left(\hat{\SSq}_t\right),,\cdots, \mathscr{F}_{\ell}\left(\hat{\SSq}_t\right)]$  to get  $\hat{\LLq}_{t}$ and $\hat{\SSq}_{t}$.
				\EndFor		
			\end{algorithmic}
	\end{algorithm}}

	 Our  LRL-RQTC method  successfully recovers the color video with missing entries and/or noise and achieves a better performance in both visual and numerical comparison.
%	 applies QMC algorithm during solving process and the
%	QMC algorithm has been proved which solution is exactly \cite{jia2018quaternion}. 
%	However, the algorithm costs a long time because of quaternion representation. Analyzing each step carefully, we can find the function 'approxQ' in the algorithm. From \eqref{e:approxq}, this function should
		 The algorithm uses  '$\texttt{approxQ}$' function \eqref{e:approxq} which is required to calculate all singular triplets of a quaternion matrix involving heavy computing cost at each iteration. We need to further increase the speed of computing. Moreover, in the practical, we can calculate the partial singular triplets and set a threshold value  to reduce time. To overcome these difficulties, we  build superior SVD solvers for the RQTC problem. 
%	To balance between computing time and recovery accuracy, a fast approximate QSVD method is needed to replace traditional QSVD, thus 
%	What we need to do is building a superior SVD solvers for the RQTC problem. 
	 We design  an implicit restarted multi-symplectic block Lanczos bidiagonalization  acceleration algorithm for quaternion SVD computation.  We will show the details of the process in the supplementary material.

	\section{Numerical Examples}\label{sec:ex}
	We will carry out various experiments in this part to demonstrate the usefulness of our low-rank algorithms for robust quaternion tensor completion (RQTC and LRL-RQTC). Below we conduct experiments on color images and videos from datasets, respectively. The level of the noise is denoted
	by $\gamma = a/(n_1n_2n_3) $, where the size of each object is $n_1 \times n_2 \times n_3$ and $a$ pixels are corrupted with uniform distribution noise.
	The level of missing entries is denoted by $(1-\rho) = 1-|\Omega|/(n_1n_2n_3)$, where $\Omega$  is the set of pixels we can observe and is also randomly selected. 
	All computations were carried out in MATLAB version R2020a on a  computer with Intel(R) Xeon(R) CPU E5-2630 @ 2.40Ghz processor and 32 GB memory.

	\begin{example}
		(The Effect of 2DQPCA Technology in Color Image Inpainting)
	\end{example}
	In this example, we compare the recoverable performance on color images with LRL-RQTC and NSS-QMC \cite{jia2022quaternion} methods. Peak signal-to-noise ratio (PSNR) and structural similarity index measure (SSIM) \cite{wbss2004} are used to assess the  quality. The levels of missing entries and  noise are considered as follows: $(1-
	\rho, \gamma) = (50\%, 10\%),(50\%, 20\%), (0\%, 10\%)$.
%	Some parameters in this experiment are setted as follows.
%	A noise is independently and randomly
%	added into $p$ pixel locations
%	of color images. 
	The tolerance  $\delta =10^{-4}$, the patch size of   two methods is both $16\times 16$ and the maximum number of iterations is $500$.
	
	Numerical results are displayed in Table \ref{Ex2_pca_nss} and restored images are shown in Fig \ref{fig:nssqmcVSllrpqmc}. In bold, the best values are highlighted in the table. We can find  that LRL-RQTC gets higher values both in PSNR and SSIM in Table \ref{Ex2_pca_nss}. That means the inpainting effect of LRL-RQTC is better than NSS-QMC from a numerical point of view. Let us see the details in Fig \ref{fig:nssqmcVSllrpqmc_cut}. On the first line, we can find  that the texture of the pepper is clearer in the (d)th column than the (c)th. In other words, the image restored by LRL-RQTC is better than NSS-QMC.
	On the second line, it is obvious that the edge of the woman's face restored by LRL-RQTC is more stereoscopic and the five sense organs are  clearer than the images restored by NSS-QMC. Besides, the texture information is also well preserved, such as the bottom left of the scarf.

	\begin{table}[!h]
		\tabcolsep 0pt \caption{RESULTS WITH RESPECT TO PSNR AND SSIM ON  IMAGES}\label{Ex2_pca_nss} \vspace*{-12pt}
		\begin{center}
			\def\temptablewidth{0.5\textwidth}
			{\rule{\temptablewidth}{0.5pt}}
			\begin{tabular*}{\temptablewidth}{@{\extracolsep{\fill}}llccc}
				\multirow{1}*{\quad \quad Images}& $(1-\rho,\gamma)$&\multirow{1}*{Methods} &  PSNR& SSIM \\ %l& Iter &\ \ 
				\hline

				\multirow{4}{*}{
					$ \begin{array}{c}{\rm Pepper}\\
						(512\times 512)\\
					\end{array}$
				}      &          \multirow{2}{*}{$(50\%,10\%)$ }
				& NSS-QMC &  33.0979     &0.9139   \\
				& & LRL-RQTC &\bf{33.2769}
				&\bf{ 0.9174}\\ 
				&
				\multirow{2}{*}{$(0\%,10\%)$}
				& NSS-QMC & 35.0680   
				&0.9442 \\
				& & LRL-RQTC &\bf{35.3002}
				& \bf{0.9507}\\

				\hline
				
				\multirow{4}{*}{
					$ \begin{array}{c}{\rm Barbara}\\
						(256\times 256)\\
					\end{array}$
				}      &          \multirow{2}{*}{$(50\%,20\%)$ }
				& NSS-QMC & 29.3625      &0.8941 \\
				& & LRL-RQTC &\bf{29.6978}
				&\bf{0.8991}\\ 
				&
				\multirow{2}{*}{$(0\%,10\%)$}
				& NSS-QMC & 32.1652     &0.9457\\
				& & LRL-RQTC &\bf{33.0635}
				&\bf{0.9577}\\

			\end{tabular*}
			{\rule{\temptablewidth}{0.5pt}}
		\end{center}
	\end{table}

	\begin{figure}
		\centering
		\includegraphics[width=0.45\textwidth,height=0.55\textwidth]{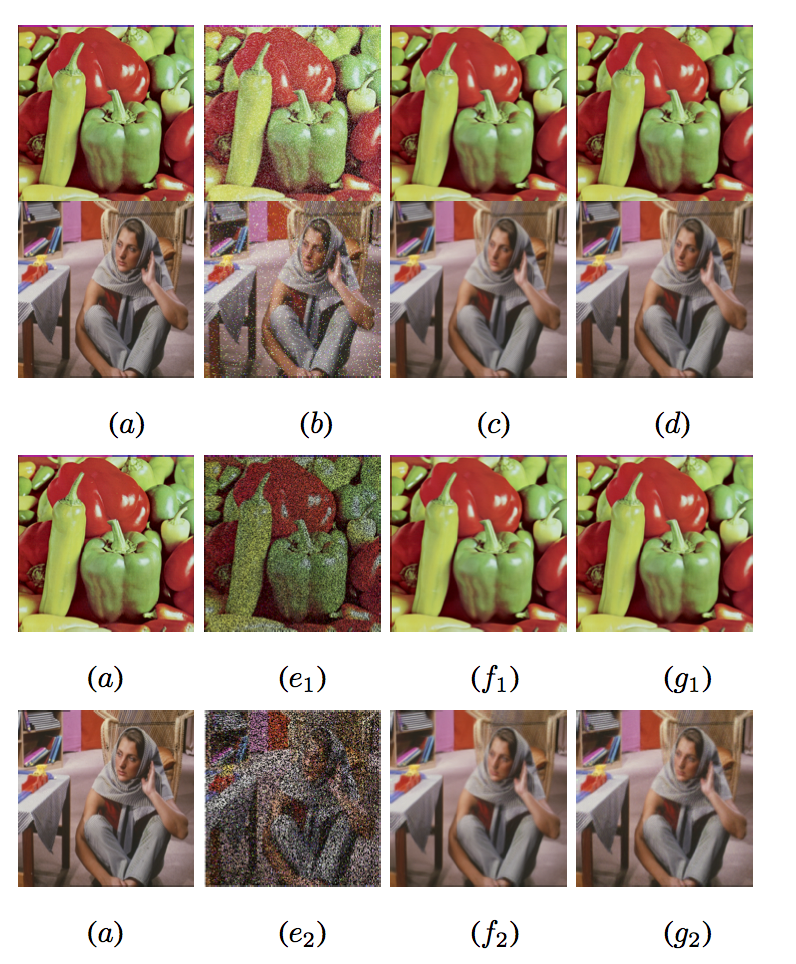}
		\caption{Inpainting results with different level:($a$) original, ($b$) ($1-\rho,\gamma$) = $(0,10\%)$, ($e_1$) ($1-\rho,\gamma$) = $(50\%,10\%)$, ($e_2$) ($1-\rho,\gamma$) = $(50\%,20\%)$, ($c$,$f_1$,$f_2$)  NSS-QMC, ($d$,$g_1$,$g_2$) LRL-RQTC}.
		\label{fig:nssqmcVSllrpqmc}
	\end{figure}

	\begin{figure}
		\centering
				\includegraphics[width=0.45\textwidth,height=0.30\textwidth]{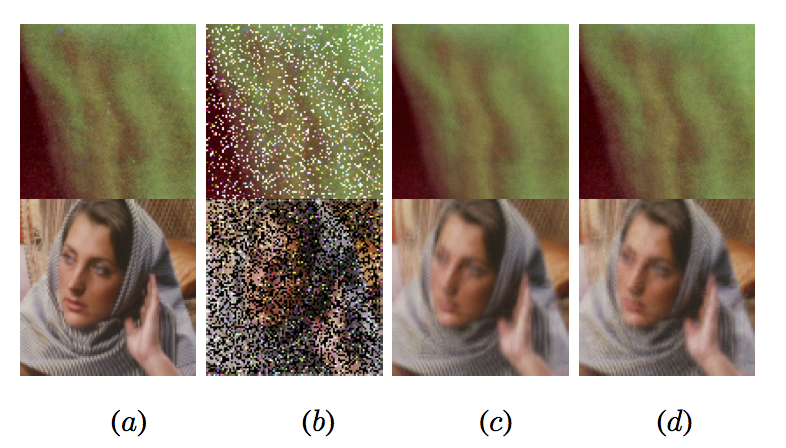}
		\caption{Enlarged parts of the first and forth row in Fig \ref{fig:nssqmcVSllrpqmc}: (a) the original,
			(b) the observations, (c) NSS-QMC, (d) LRL-RQTC. }.
		\label{fig:nssqmcVSllrpqmc_cut}
	\end{figure}

	\begin{example}
		(Robust Color Video Recovery)
	\end{example}
	In this example, we  compare the proposed RQTC and LRL-RQTC methods with other modern techniques in order to demonstrate the effectiveness and superiority of our methods in robust color video recovery. 
%	The testing video is from the ‘videoSegmentationData’ database
%	in \cite{fmkty09}.  
	We choose the color video ‘DO01$\_$013’ in ‘videoSegmentationData’ database
	 \cite{fmkty09}  of size $243 \times 256$ and use a 3-mode quaternion tensor to represent it. We also compare the performance of models under different missing entries and noise levels. 
	
%	A noise is independently and randomly
%	added into $p$ pixel locations
%	of each frame of color video. The level of noise components is denoted
%	by $\gamma = p/(mn) $, where the size of frame is $m \times n$.
%	Let $\Omega$  be the set of observed entries which are generated
%	randomly and let $\rho = |\Omega|/(mn)$ denote the percentage of observed entries. 
	We set $(1-\rho, \gamma) = (10\%, 10\%),(20\%, 10\%)$.
	The stopping criteria  is $\delta =10^{-4}$ and  suitable parameters are chosen to obtain the best results of each model. We compare RQTC and LRL-RQTC  with  the other six robust quaternion tensor completion methods.
	\begin{itemize}
		\item  t-SVD (Fourier) \cite{tsvdfourier2019}: Tensor SVD using Fourier transform.
		\item  TNN \cite{zatnn2017}: Recover  3-D arrays with a low tubal-rank tensor.
		\item OITNN-O \cite{wzj2022}: Recover tensors by  new tensor norms: Overlapped Orientation Invariant Tubal.
		\item OITNN-L \cite{wzj2022}: Recover tensors by  new tensor norms: Latent Orientation Invariant Tubal Nuclear Norm.
		\item QMC \cite{jia2018quaternion}: Consider each frame of color video as a color image.
		\item TNSS-QMC  \cite{jia2022quaternion}: Use distance function to find similar patches.
		
	\end{itemize}
	For t-SVD (Fourier), TNN, OITNN-O, OITNN-L, QMC, and TNSS-QMC, we strictly follow the parameters set in the article.
	For LRL-RQTC, we fix each window of size $27\times 32$ and only act on the frontal slice, i.e. $ [\alpha_1,\alpha_2,\alpha_3] =
	[0,0,1]$. The maximum number of iterations is $100$, each patch size is $11 \times 11$.
	For  RQTC, the weighted vector $\alpha = [\alpha_1,\alpha_2,\alpha_3] =[0.8, 0.1,0.1]$.
	
	The detailed comparisons of the ‘DO01$\_$013’ dataset are given in Table \ref{ex3:video}. We  display the comparative results on  ten frames randomly choosen from color video of two situations. From 
	the last row of Table  \ref{ex3:video}, 
	with respect to PSNR and SSIM, the suggested LRL-RQTC  outperforms all other competitors on the average recovery results.
	 Our LRL-RQTC exceeds
	the second-best  by $0.7$ dB PSNR value on average. In particular, from the sub-table (a), the PSNR value of the restored frames of video can be increased by nearly 1dB under $10\%$ noise and missing entries level. From the sub-table (b), although on individual frames our LRL-RQTC  is not as good as OITNN-L, the average result of our method is higher and the result of OITNN-L is not stable. Furthermore, different from patch-based methods such as TNSS-QMC and LRL-RQTC, our proposed   RQTC recovers color video on the whole. Despite not reaching the highest PSNR and SSIM values,  RQTC  outperforms tsvd (Fourier), TNN, OITNN-O, and QMC  which  also processes the entire video without patching. It is worth noting that  RQTC outperforms QMC  by $3$ dB  PSNR value and are only slightly less than the other patch's method. The  RQTC can also be handled very well in small details and the global approach is much faster than the patch-based approach, so we believe that there is much potential to improve  RQTC.
	
	In Fig. \ref{fig:videohorsenan01nois01} and Fig. \ref{fig:videohorsenan02nois01}, we present the visual results on one frame under two conditions, where it can be seen that recovered video frames generated through the proposed  RQTC and LRL-RQTC contain more details  and more closely approximate the color of the initial frames.
	For instance, we  observe the detailed features of the restored grass carefully. From the enlarged part (2nd row) of Fig. \ref{fig:videohorsenan01nois01},  color of shades and density of the grass in (c)th (h)th and (j)th columns have a
	clearer structure and look more realistic but the horse of (c)th column  are not restored clearer. This means this video restored by our LRL-RQTC and RQTC are better. Although the PSNR and SSIM values of RQTC  are not higher than TNSS-QMC, the recovery of certain details is superior to it. It shows that both RQTC and LRL-RQTC outperform in most cases. 
	From the enlarged part (2nd row) of Fig. \ref{fig:videohorsenan02nois01}, the white point of the horse's head is restored better by LRL-RQTC, it looks  brighter and completer. The proposed LRL-RQTC  uses the idea of clustering and finding similar patches adaptively by learning technology. Patches by learning technology can form low-rank matrices and make sure the effect of recovery is better. In particular, it is not necessary to calculate all singular values during the process.
	As a result, with the proper amount of patches, the suggested LRL-RQTC will be substantially more efficient.

	\begin{table*}
		\caption{
			PSNR AND SSIM VALUES BY DIFFERENT METHODS ON THE COLOR VIDEO.}
%			estoring results for color video corrupted by sample losing and uniformly distributed noise.}
		\begin{center}
		 (1) $(1-
			\rho, \gamma) = (10\%, 10\%)$
		\end{center}
		\label{ex3:video}
		\begin{center}
			\tabcolsep=0.15cm
			\renewcommand{\arraystretch}{1.2} \vskip-3mm {\fontsize{8pt}{\baselineskip}%
				\selectfont
				\scalebox{0.99}{
					\begin{tabular}{|c|cc|cc|cc|cc|cc|cc|cc|cc|}
						\hline
						Number of        & \multicolumn{2}{c|}{t-SVD (Fourier)}
						&\multicolumn{2}{c|}{TNN} 
						&\multicolumn{2}{c|}{OITNN-O}
						&\multicolumn{2}{c|}{OITNN-L}
						&\multicolumn{2}{c|}{QMC}
						&\multicolumn{2}{c|}{ RQTC }
						&\multicolumn{2}{c|}{TNSS-QMC } &\multicolumn{2}{c|}{LRL-RQTC} 
				         \\
						frames      &    PSNR&SSIM &PSNR&SSIM &PSNR&SSIM&   PSNR&SSIM &PSNR&SSIM &	PSNR&SSIM     &  PSNR&SSIM &PSNR&SSIM \\ \hline\hline
						1
						&  32.76 & 0.9039 & 29.87 & 0.8938 &30.33 & 0.8966 &  38.33 & 0.9764 
						&  33.51 & 0.9303   
						&  {36.83}  & {0.9525}&
						\underline{38.79} & \underline{0.9709}   &   {\bf 40.22}& {\bf  0.9785} \\ \hline
						
						2
						& 31.71 & 0.8888 &  30.50 &0.9105 &  33.95 & 0.9500 &  \underline{39.73} &\underline{0.9799} &  33.63 & 0.9336   &   {37.19}  & {0.9544} &
						{39.56} & 0.9759 &   {\bf 40.94}  &  {\bf 0.9814} \\ \hline
					
						3
						&  32.08 & 0.8957 &  31.09 & 0.9245 &  35.59 & 0.9628 &  39.30 & \underline{0.9780} &  33.84 & 0.9354   &   {37.36}  & {0.9553} &
						\underline{40.09} & {0.9776}   &   {\bf 41.35} &  {\bf 0.9823} \\ \hline
						
						4
						& 32.57 & 0.8982 &  31.02 & 0.9254 &  35.99 & 0.9688 &  39.09 & \underline{0.9793} &33.83 & 0.9361  &  {37.32}  & {0.9560} &
						\underline{39.75} & {0.9769}   &   {\bf 40.92}  &  {\bf 0.9810} \\ \hline
						
						5
						& 32.59 & 0.9025 & 30.43 & 0.9167 &  35.64 & 0.9692 &  38.30 & \underline{0.9778} &  33.86 & 0.9359   &   {36.35} &  {0.9524} &
						\underline{39.32} & {0.9745}   &   {\bf 39.84}  &  {\bf 0.9771} \\ \hline
						6
						& 32.64 & 0.9014 &   29.72 & 0.9045 &  35.22 & 0.9668 &  37.64 & 0.9763 &  33.76 & 0.9361   &  {35.36}  & {0.9495}&
						\underline{37.59} & \underline{0.9670}   &   {\bf 37.97}& {\bf  0.9716}  \\ \hline
							7
						& 31.60 & 0.8895 &  29.46 & 0.8899 &  34.58 & 0.9622 &  37.41 & \underline{0.9765} &  33.85 & 0.9345   &  {35.64}  & { 0.9527}&
						\underline{38.32} & {0.9715}   &   {\bf 39.17}& {\bf  0.9749}  \\ \hline
						8
						& 32.32 & 0.8967 &  29.73 & 0.8907 &  33.72 & 0.9572 &  36.29 & 0.9732  &  34.16 & 0.9367   &  {35.79}  & {0.9539}&
						\underline{39.15} & \underline{0.9751}   &   {\bf 39.93}& {\bf  0.9778}  \\ \hline
						9
						& 32.54 & 0.8997 &  29.41 & 0.8870 & 32.75 & 0.9477 &  36.03 & 0.9725 &  34.09 & 0.9357   &  {35.87}  & {0.9535}&
						\underline{39.40} & \underline{0.9758}   &   {\bf 39.99}& {\bf  0.9783} \\ \hline
						10
						& 32.58 & 0.9029 &  28.37 & 0.8679 & 28.56 & 0.8733 &  34.69 & 0.9683 &  34.18 & 0.9362   &  {35.10}  & {0.9506}&
						\underline{38.61} & \underline{0.9733}   &   {\bf 39.12}& {\bf  0.9757}  \\ \hline
						average 
						& 32.34 & 0.8979  & 29.96 & 0.9011 &  33.63 & 0.9455 &  37.68 & \underline{0.9758}  &  33.8 & 0.9351 &  {36.28}  & {0.9531}&
						\underline{39.06} & {0.9739}   &   {\bf 39.95}& {\bf  0.9778}  \\ \hline
				\end{tabular}}
			} 
		\end{center}
		
		\begin{center}
			(2) $(1-
			\rho, \gamma) = (20\%, 10\%)$
		\end{center}
		\label{tab:video}
		\begin{center}
			\tabcolsep=0.15cm
			\renewcommand{\arraystretch}{1.2} \vskip-3mm {\fontsize{8pt}{\baselineskip}%
				\selectfont
				\scalebox{0.99}{
				\begin{tabular}{|c|cc|cc|cc|cc|cc|cc|cc|cc|}
					\hline
					Number of        & \multicolumn{2}{c|}{t-SVD (Fourier)}
					&\multicolumn{2}{c|}{TNN} 
					&\multicolumn{2}{c|}{OITNN-O}
					&\multicolumn{2}{c|}{OITNN-L}
					&\multicolumn{2}{c|}{QMC}
					&\multicolumn{2}{c|}{ RQTC }
					&\multicolumn{2}{c|}{TNSS-QMC } &\multicolumn{2}{c|}{LRL-RQTC} 
					\\
					frames      &    PSNR&SSIM &PSNR&SSIM &PSNR&SSIM&   PSNR&SSIM &PSNR&SSIM &	PSNR&SSIM     &  PSNR&SSIM &PSNR&SSIM \\ \hline\hline
						1
						& 31.96 &0.8927 &  29.79 & 0.8920&  30.32 & 0.8962&  37.32 & \underline{0.9662} &  32.94 & 0.9165   &  {35.64}  & {0.9369}&
						\underline{37.77} & {0.9657}   &   {\bf 38.32}& {\bf  0.9698}  \\ \hline	
						
						2
						& 31.13 & 0.8788 &  30.42 & 0.9091&   33.97 & 0.9498 & {\bf 39.84} & {\bf 0.9801} &  32.99 & 0.9202   &   {36.05}  & {0.9400} &
						{38.75} & {0.9729}   &   \underline{39.40}  &  \underline{0.9761} \\ \hline
						
						3
					    & 31.41 & 0.8834 & 31.03 & 0.9237 &  35.63 & 0.9627 &  \underline{39.41} & {\bf  0.9881} 	&  33.10 & 0.9211   &   {36.15}  & {0.9406} &
						{39.05} & {0.9737}   &   {\bf 39.82} &  \underline{0.9776} \\ \hline
						
						4
						& 31.98 & 0.8912 &  30.98 & 0.9249&  36.00 & 0.9686 &  \underline{39.12} & {\bf 0.9794}  &32.95 & 0.9233  &  {36.05}  & {0.9415} &
						{38.31} & {0.9724}   &   {\bf 39.18}  &  \underline{ 0.9759} \\ \hline
					    5
						& 31.99 & 0.8907 &  30.41 & 0.9166&  35.65 & 0.9693 &  {\bf 38.32} & {\bf 0.9780} &  32.83 & 0.9223   &   {35.41} &  {0.9388} &
						{37.09} & {0.9666}   &   \underline{ 37.45}  &  \underline{0.9692} \\ \hline
						6 
						& 31.97  &0.8895&  29.74 & 0.9052&  35.25 & 0.9165 &  36.69 & {\bf 0.9666} &  33.00 & 0.9215   &  {34.45}  & {0.9393}
						& \underline{37.12} & {0.9618}   &   {\bf 37.30}& \underline{0.9655}  \\ \hline
						7
						& 31.04 & 0.8743 &  29.47 & 0.8901 &  34.57 & 0.9619 &  37.37 & 0.9663 &  32.99 & 0.9201   &  {34.74}  & {0.9425}
						& \underline{37.71} & \underline{0.9652}   &   {\bf 38.18}& {\bf  0.9681}   \\ \hline
						8
						& 31.53 & 0.8854 &  29.69 & 0.8901&  33.68 & 0.9567  &  36.26 & 0.9630 &  33.08 & 0.9203   &  {34.92}  & {0.9431}
						& \underline{38.24} & \underline{0.9683}   &   {\bf 38.74}& {\bf  0.9707}   \\ \hline
						9
						& 31.82 & 0.8905 & 29.37 & 0.8861 & 32.72 &  0.9476 &  35.99 & 0.9622  &  32.82 & 0.9190   &  {35.89}  & {0.9429}
						& \underline{38.50}& \underline{0.9694}   &   {\bf 38.89}& {\bf  0.9715}  \\ \hline
						
						10
						& 32.15 & 0.8947 &  28.33 & 0.8672&  28.54 & 0.8733 &  34.69 & 0.9586  &  32.86 & 0.9204   &  {34.30}  & {0.9401}
						& \underline{37.79} & \underline{0.9665}   &   {\bf 37.98}& {\bf  0.9680}  \\ \hline
						average 
						& 31.70 & 0.8871& 29.92 & 0.9005 & 33.62 & 0.9402& 37.50&\underline{0.9708}&32.96& 0.9204&35.36&0.9406&\underline{38.03}&0.9683&{\bf 38.53}&{\bf 0.9712} \\ \hline
				\end{tabular}}
			}
		\end{center}

	\end{table*}
	
	\begin{figure*}
	\centering
			\includegraphics[width=0.96\textwidth,height=0.3\textwidth]{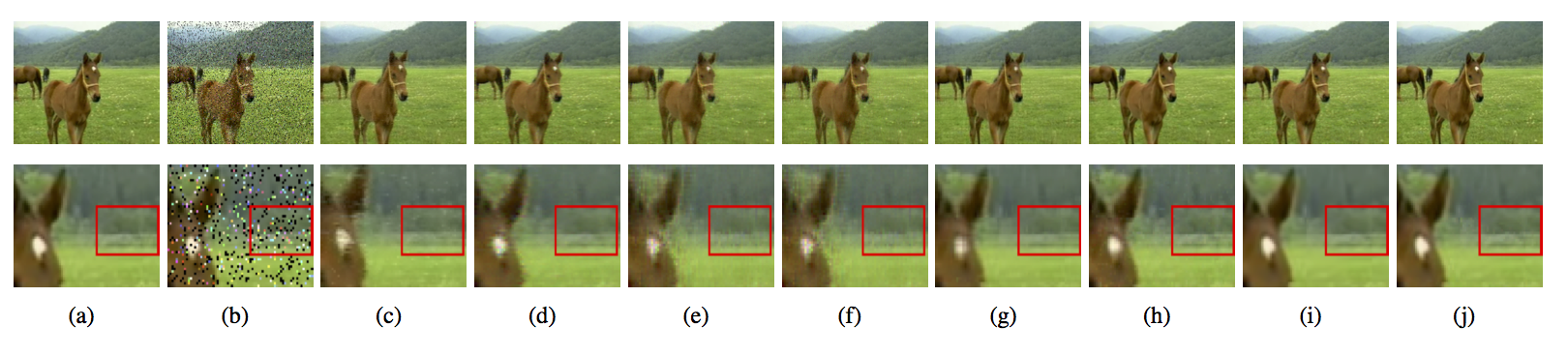}
	
	%\vspace{-0.35cm}
	
	\caption{ One frame (1st row) and their enlarged parts (2nd row) by eight methods with 10\% unobserved and 10\% corrupted entries. (a) the original. (b)
		the observations. (c) t-SVD (Fourier). (d) TNN. (e) OITNN-O. (f) OITNN-L. (g) QMC. (h)  RQTC. (i) TNSS-QMC. (j) LRL-RQTC.}
	\label{fig:videohorsenan01nois01}
\end{figure*}
	
\begin{figure*}
		\centering
	\includegraphics[width=0.96\textwidth,height=0.3\textwidth]{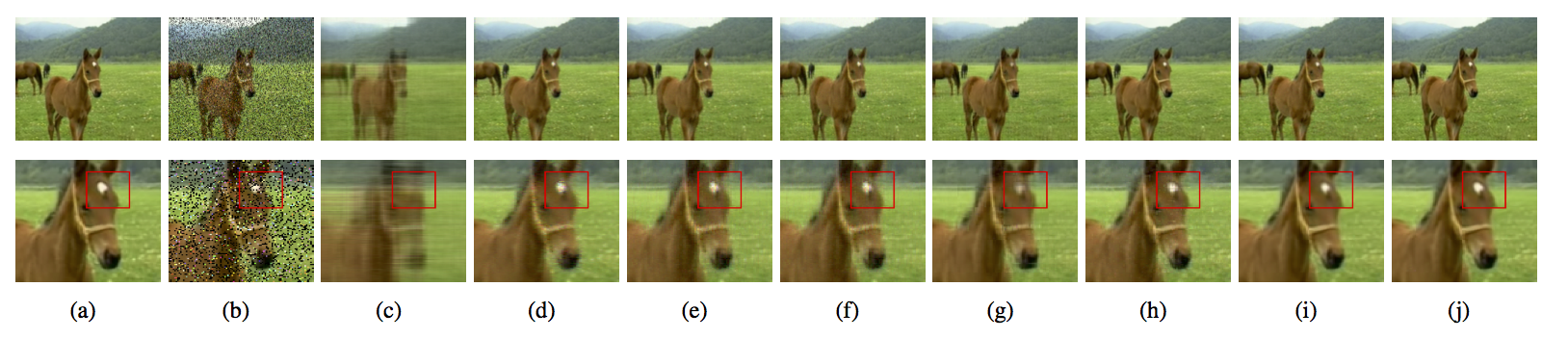}
		
		%\vspace{-0.35cm}
		
		\caption{ One frame (1st row) and their enlarged parts (2nd row) by eight methods  with 20\% unobserved and 10\% corrupted entries. (a) the original images. (b)
			the observations. (c) t-SVD (Fourier). (d) TNN. (e) OITNN-O. (f) OITNN-L. (g) QMC. (h)  RQTC. (i) TNSS-QMC. (j) LRL-RQTC.}
		\label{fig:videohorsenan02nois01}
	\end{figure*}

	\begin{example}
	(Color Video Completion)
\end{example}
In this example, we deal with the color video completion problem, which involves filling in missing pixel values from a partly unobserved video. Assume the missing pixels are distributed randomly across  RGB three channels.
%various degrees of pixel loss of each frame, that is, the missing affects all the RGB channels at the same time rather than only parts of the channels,
%making it to be a more challenging problem. 

We compare the proposed  RQTC and LRL-RQTC  with representative models for color video completion: t-SVD (Fourier)\cite{tsvdfourier2019}, t-SVD (data)\cite{smztsvddata2020} (using a transform tensor
singular value decomposition based on a unitary transform matrix), OITNN-O\cite{wzj2022}, OITNN-L\cite{wzj2022}, LRQTC\cite{miao2020qtc} (quaternion tensor completion) , QMC\cite{jia2018quaternion} and TNSS-QMC\cite{jia2022quaternion}. 

The color
video dataset is also chosen from the ‘videoSegmentationData’ database
in \cite{fmkty09}.  We choose the color videos ‘BR130T’, ‘DO01$\_$030’, and 'DO01$\_$013' of size  $288 \times 352$, which are denoted as ‘bird’, ‘flower’, and ‘horse’,  respectively, and the videos can be expressed as 3-mode quaternion tensors. Their parameter values are empirically chosen to produce the greatest performance and are fixed in all testing for a fair comparison.

Different missing entries and noise levels are considered: $(1-
\rho, \gamma) = (50\%, 0\%),(80\%, 0\%),(85\%, 0\%)$.
In the proposed LRL-RQTC, we fix each window of size $36 \times 44$ and only act on the frontal slice, i.e.$ [\alpha_1,\alpha_2,\alpha_3] =
[0,0,1]$. The maximum number of iterations is $100$, each patch size is $8 \times 8$. For  RQTC, the weighted vector $\alpha = [\alpha_1,\alpha_2,\alpha_3] =
[0.8, 0.1,0.1]$.

For quantitative comparison, PSNR and SSIM values on the `bird', `flower', and `horse' videos are reported in Fig. \ref{f:birdcomplection05}, Fig \ref{f:flowercomplection08} and Table \ref{tab:horsecomplection08}. The corresponding visual examples are shown in
Fig. \ref{fig:videobirdnan05nois00} - \ref{fig:videobirdnan08nois00}, Fig. \ref{fig:videoflowernan08nois00} - \ref{fig:videoflowernan085nois00},  and Fig. \ref{fig:videohorsenan08nois00} - \ref{fig:videohorsenan09nois00} for qualitative evaluation. In the four histograms (Fig. \ref{f:birdcomplection05} and Fig \ref{f:flowercomplection08}), we can find that PSNR and SSIM in earthy yellow on the right-hand side are higher than the other. In Table \ref{tab:horsecomplection08}, we present the PSNR and SSIM values of random ten frames of the 'horse' video to provide additional information. The standout performance is bolded and the underlined data indicates sub-optimal. It is clear that the proposed LRL-RQTC   improves PSNR and SSIM values considerably. The PSNR values obtained by the LRL-RQTC  on three color video data improve by almost $3$ dB when compared to the second-best method and are always at the highest value, demonstrating the proposed method's superiority over the other good approaches presently in use. That implies our LRL-RQTC   is the best among the ten methods and suitable for different types of video. In particular, the PSNR values achieved by the LRL-RQTC   on the 'flower' color video data increase by roughly $4$ dB when compared to the second-best method. As displayed in the graph, the proposed LRL-RQTC   outperforms the rival method in terms of visual quality. It can be seen that our LRL-RQTC  captures the detail of each frame properly, which implies that learning technology to find similar patches is superior. 
From Fig. \ref{fig:videobirdnan05nois00}, the
videos restored by LRL-RQTC ($\ell$th column) are clearer and
show more details of the bird's wings as well as the foliage. Moreover, videos recovered by  RQTC ($k$th column) also restores details of the original video. Both LRL-RQTC and  RQTC   are superior to the other methods in respect of local features' recovery.
In Fig. \ref{fig:videoflowernan08nois00} - Fig. \ref{fig:videohorsenan09nois00}, the details of flower petals, grass, the tails of horse and so on recovered by LRL-RQTC   are closer to the original frame. We present some of them and use the red box on the images for details. 

For computing time, we display the result in Table \ref{cpu}. Because of TNSS-QMC and our LRL-RQTC are based on patch ideas, per-patch based operations can be paralleled, so we only show the average time of each window. From Table \ref{cpu}, we can observe that our LRL-RQTC  gets better results without taking too long. It slightly less than OITNN-L which is also gets better recovery effect.

Remember  that the LRL-RQTC model  designed by using a learning perspective framework  searches for similar structures adaptively and forms a   low-rank structure based on the patch idea. This model is made feasible by the 2DQPCA-based classification function that characterizes the low-rank information from subspace structure after projection. Numerical results indicate that the LRL-RQTC model    has clearly made a significant increase in its ability to find high-dimensional information in low-dimensional space and accurately identify the delicate correlations between the observed and unknown values.

\begin{figure}%[H]
	\begin{center}
			\includegraphics[width=0.46\textwidth,height=0.25\textwidth]{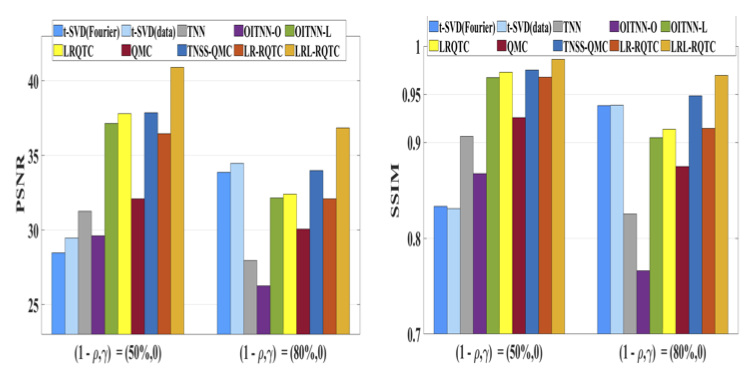}
	\end{center}
	\vskip-15pt
	\caption{ PSNR and SSIM comparisons on `bird' color video.}\label{f:birdcomplection05}
\end{figure}

\begin{figure}
	\centering
	\includegraphics[width=0.46\textwidth,height=0.3\textwidth]{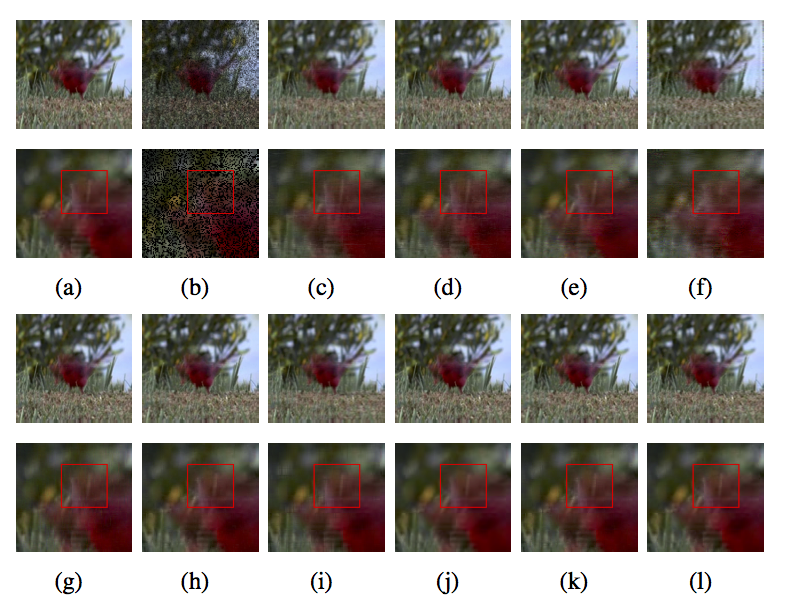}
	
	%\vspace{-0.35cm}
	
	\caption{ One frame (1st ans 3rd row) and their enlarged parts (2nd and 4th row) by ten methods with 50\% unobserved entries. (a) the original. (b)
		the observations. (c) t-SVD (Fourier). (d) t-SVD (data). (e) TNN. (f) OITNN-O. (g) OITNN-L. (h) LRQTC. (i) QMC. (j) TNSS-QMC. (k)  RQTC. ($\ell$) LRL-RQTC.}
	\label{fig:videobirdnan05nois00}
\end{figure}
\begin{figure}
	\centering
	\includegraphics[width=0.46\textwidth,height=0.3\textwidth]{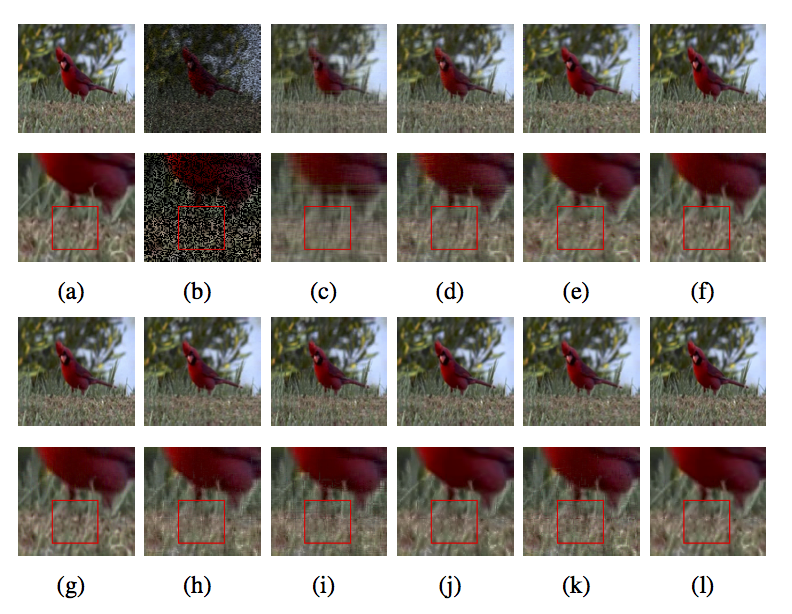}
	
	%\vspace{-0.35cm}
	
    \caption{ One frame (1st ans 3rd row) and their enlarged parts (2nd and 4th row) by ten methods with 85\% unobserved entries. (a) the original. (b)
    	the observations. (c) t-SVD (Fourier). (d) t-SVD (data). (e) TNN. (f) OITNN-O. (g) OITNN-L. (h) LRQTC. (i) QMC. (j) TNSS-QMC. (k)  RQTC. ($\ell$) LRL-RQTC.}
	\label{fig:videobirdnan08nois00}
\end{figure}

\begin{figure}%[H]
	\begin{center}
			\includegraphics[width=0.46\textwidth,height=0.26\textwidth]{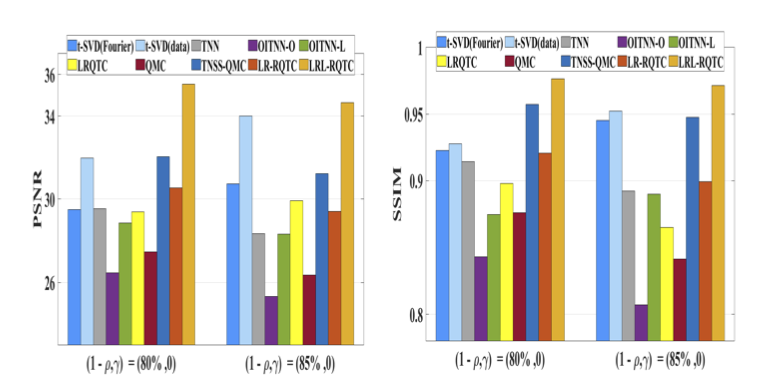}
	\end{center}
	\vskip-15pt
	\caption{ PSNR and SSIM comparisons on `flower' color video.}\label{f:flowercomplection08}
\end{figure}

\begin{figure}
	\centering
	\includegraphics[width=0.46\textwidth,height=0.3\textwidth]{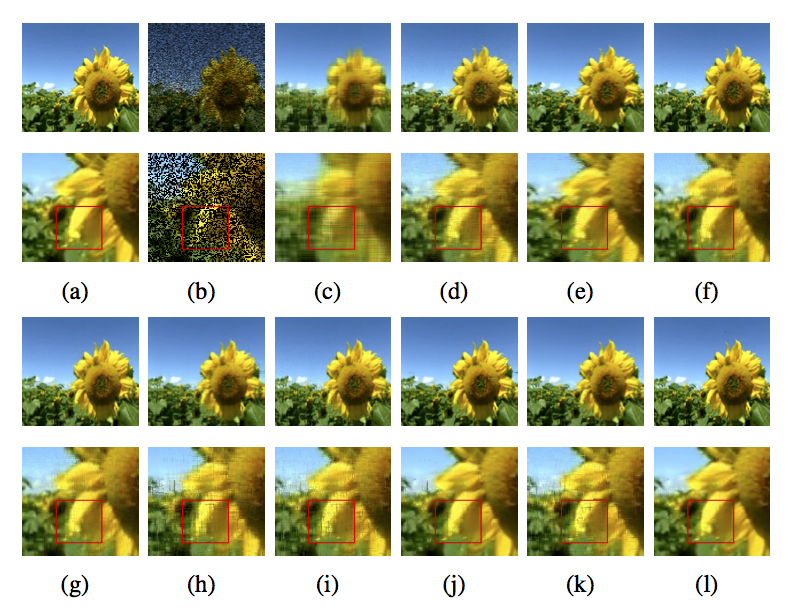}
	
	%\vspace{-0.35cm}
	
	\caption{ One frame (1st ans 3rd row) and their enlarged parts (2nd and 4th row) by ten methods with 80\% unobserved entries. (a) the original. (b)
		the observations. (c) t-SVD (Fourier). (d) t-SVD (data). (e) TNN. (f) OITNN-O. (g) OITNN-L. (h) LRQTC. (i) QMC. (j) TNSS-QMC. (k)  RQTC. ($\ell$) LRL-RQTC.}
	\label{fig:videoflowernan08nois00}
\end{figure}
\begin{figure}
	\centering
		\includegraphics[width=0.46\textwidth,height=0.3\textwidth]{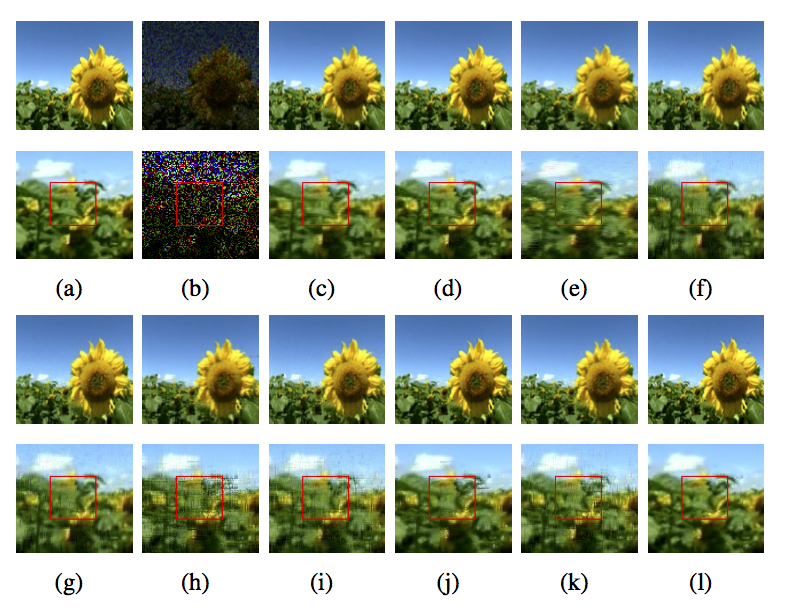}
	%\vspace{-0.35cm}
	
	\caption{ One frame (1st ans 3rd row) and their enlarged parts (2nd and 4th row) by ten methods  with 85\% unobserved entries. (a) the original. (b)
		the observations. (c) t-SVD (Fourier). (d) t-SVD (data). (e) TNN. (f) OITNN-O. (g) OITNN-L. (h) LRQTC. (i) QMC. (j) TNSS-QMC. (k)  RQTC. ($\ell$) LRL-RQTC.}
	\label{fig:videoflowernan085nois00}
\end{figure}
\begin{table*}
	\caption{
		PSNR AND SSIM VALUES BY DIFFERENT METHODS ON  ``HORSE'' COLOR VIDEO.}
	\begin{center}
		(1)  $(1-
		\rho, \gamma) = (80\%, 0)$
	\end{center}
	\	\label{tab:horsecomplection08}
	\begin{center}
		\tabcolsep=0.1cm
		\renewcommand{\arraystretch}{2} \vskip-3mm {\fontsize{12pt}{\baselineskip}%
			\selectfont
			\scalebox{0.6}{
				\begin{tabular}{|c|cc|cc|cc|cc|cc|cc|cc|cc|cc|cc|}
					\hline
					Number of        & \multicolumn{2}{c|}{t-SVD (Fourier)}& \multicolumn{2}{c|}{t-SVD (data)}
					&\multicolumn{2}{c|}{TNN} 
					&\multicolumn{2}{c|}{OITNN-O}
					&\multicolumn{2}{c|}{OITNN-L}
					&\multicolumn{2}{c|}{LRQTC}
					&\multicolumn{2}{c|}{QMC}
					&\multicolumn{2}{c|}{TNSS-QMC }
					&\multicolumn{2}{c|}{ RQTC }
					&\multicolumn{2}{c|}{LRL-RQTC} 
					\\
					frames      &    PSNR&SSIM &PSNR&SSIM &PSNR&SSIM&   PSNR&SSIM &PSNR&SSIM &	PSNR&SSIM     &  PSNR&SSIM &PSNR&SSIM&    PSNR&SSIM &    PSNR&SSIM 
					\\ \hline\hline
					1
					&   {29.29} & {0.8817}
					&   {33.80} & {0.9435}
					&   {33.74} &  \underline{0.9521}
					&   {31.43} & {0.9016}
					&   {31.38} & {0.9007}
					&  {32.11}  & {0.8815}&
					29.77 & 0.8253   &   \underline{35.21}& {0.9408}	
					&  33.02 & 0.8941   &  {\bf 37.06}  & {\bf 0.9599}
					\\ \hline
					2
					&   {28.47} & {0.8571}
					&   {33.19} & {0.9339}
					&   {34.72} &  \underline{0.9622}
					&   {33.92} & {0.9369}
					&   {33.89} & {0.9364}
					&   {32.10}  & {0.8843} &
					29.71 & 0.8301   &   \underline{35.67}  &  {0.9453}
					&  33.18 & 0.8973   &  {\bf 37.46}  & {\bf 0.9629}
					\\ \hline
					3	
					&   {28.89} & {0.8707}
					&   {33.66} & {0.9405}
					&   {35.39} &  \underline{0.9661}
					&   {34.59} & {0.9440}
					&   {34.56} & {0.9436}
					&   {32.31}  & {0.8851} &
					30.09 & 0.8353   &   \underline{36.30} &  {0.9505}
					&  33.36 & 0.9000   &  {\bf 37.68}  & {\bf 0.9637}
					\\ \hline
					4	
					&   {29.54} & {0.8800}
					&   {33.95} & {0.9412}
					&   {34.99} &  \underline{0.9645}
					&   {34.79} & {0.9485}
					&   {34.75} & {0.9481}
					&  {32.43}  & {0.8886} &
					30.13 & 0.8380   &   \underline{36.35}  &  {0.9504}
					&  33.39 & 0.8998   &  {\bf 37.63}  & {\bf 0.9644}
					\\ \hline
					5	
					&   {29.58} & {0.8828}
					&   {33.99} & {0.9429}
					&   {34.18} &  \underline{0.9598}
					&   {34.74} & {0.9505}
					&   {34.71} & {0.9501}
					&   {32.29} &  {0.8873} &
					30.04 & 0.8385   &   \underline{35.99}  &  {0.9489}
					&  33.20 & 0.8978  &  {\bf 37.60}  & {\bf 0.9647}
					\\ \hline
					6
					&   {29.29} & {0.8786}
					&   {33.76} & {0.9409}
					&   {33.58} &  \underline{0.9539}
					&   {34.60} & {0.9484}
					&   {34.56} & {0.9480}
					&   {35.15} &  {0.8795} &
					30.11 & 0.8368   &   \underline{35.38}  &   {0.9413}
					&  32.56 & 0.8812   &  {\bf 37.27}  & {\bf 0.9621}
					\\ \hline
					7
					&   {28.30} & {0.8563}
					&   {33.03} & {0.9341}
					&   {33.12} &  \underline{0.9460}
					&   {34.16} & {0.9451}
					&   {34.12} & {0.9447}
					&   {32.08} &  {0.8784} &
					30.04 & 0.8337   &   \underline{35.16}  &  {0.9373}
					&  32.59 & 0.8800   &  {\bf 37.09}  & {\bf 0.9593}
					\\ \hline
					8
					&   {28.56} & {0.8673}
					&   {33.47} & {0.9400}
					&   {33.04} &  \underline{0.9445}
					&   {33.92} & {0.9433}
					&   {33.89} & {0.9429}
					&   {32.26} &  {0.8798} &
					30.04 & 0.8388   &  \underline{35.26}  &   {0.9375}
					&  32.69 & 0.8804   &  {\bf 37.17}  & {\bf 0.9595}
					\\ \hline
					9
					&   {29.13} & {0.8748}
					&   {33.54} & {0.9396}
					&   {33.02} &  \underline{0.9438}
					&   {33.60} & {0.9360}
					&   {33.56} & {0.9354}
					&   {32.28} &  {0.8800} &
					30.02 & 0.8341   &   \underline{35.24}  &  {0.9374}
					&  32.70 & 0.8816   &  {\bf 37.22}  & {\bf 0.9595}
					\\ \hline
					10
					&   {28.94} & {0.8752}
					&   {33.54} &  \underline{0.9426}
					&   {31.74} & {0.9289}
					&   {30.75} & {0.8940}
					&   {30.71} & {0.8932}
					&   {32.31} &  {0.8808} &
					30.08 & 0.8356   &   \underline{35.29}  &  {0.9389}
					&  32.67 & 0.8818   &  {\bf 37.33}  & {\bf 0.9601}
					\\ \hline
					average
					&   {29.04} & {0.8725}
					&   {33.59} & {0.9299}
					&   {33.75} & \underline{0.9522}
					&   {33.65} & {0.9348}
					&   {33.61} & {0.9343}
					&   {32.53} &  {0.8808} &
					30.00 & 0.8346   &   \underline{35.56}  &  {0.9428}
					&  32.94 & 0.8894   &  {\bf 37.35}  & {\bf 0.9616}
					\\ \hline
			\end{tabular}}
		} 
	\end{center}
	
	\begin{center}
		(2)  $(1-
		\rho, \gamma) = (85\%, 0\%)$
	\end{center}
	\label{tab:horsecomplection09}
	\begin{center}
		\tabcolsep=0.1cm
		\renewcommand{\arraystretch}{2} \vskip-3mm {\fontsize{12pt}{\baselineskip}%
			\selectfont
			\scalebox{0.6}{
				\begin{tabular}{|c|cc|cc|cc|cc|cc|cc|cc|cc|cc|cc|}
					\hline
					Number of        & \multicolumn{2}{c|}{t-SVD (Fourier)}& \multicolumn{2}{c|}{t-SVD (data)}
					&\multicolumn{2}{c|}{TNN} 
					&\multicolumn{2}{c|}{OITNN-O}
					&\multicolumn{2}{c|}{OITNN-L}
					&\multicolumn{2}{c|}{LRQTC}
					&\multicolumn{2}{c|}{QMC}
					&\multicolumn{2}{c|}{TNSS-QMC }
					&\multicolumn{2}{c|}{ RQTC }
					&\multicolumn{2}{c|}{LRL-RQTC} 
					\\
					frames      &    PSNR&SSIM &PSNR&SSIM &PSNR&SSIM&   PSNR&SSIM &PSNR&SSIM &	PSNR&SSIM     &  PSNR&SSIM &PSNR&SSIM&    PSNR&SSIM &    PSNR&SSIM 
					\\ \hline\hline
					1
					&   {25.99} & {0.7562}
					&   {30.42} & {0.8748}
					&   {32.65} &\underline{0.9373}
					&   {30.28} & {0.8729}
					&   {30.23} & {0.8718}
					&  {31.66}  & {0.8642}&
					30.17 & 0.8737   &  \underline{33.94}&  {0.9243}	
					&  31.72 &  0.8560  &  {\bf 37.16}  & {\bf 0.9623}
					\\ \hline
					2
					&   {25.41} & {0.7251}
					&   {29.93} & {0.8576}
					&   {33.46} & \underline{0.9486}
					&   {32.51} & {0.9127}
					&   {32.46} & {0.9118}
					&   {31.63}  & {0.8638} &
					29.96 & 0.8317   &   \underline{33.65}  &   {0.9189}
					&  31.74 & 0.8544   &  {\bf 36.93}  & {\bf 0.9595}
					\\ \hline
					3
					&   {25.62} & {0.7418}
					&   {30.19} & {0.8669}
					&   {33.98} & \underline{0.9548}
					&   {33.07} & {0.9236}
					&   {33.04} & {0.9230}
					&   {31.72}  & {0.8673} &
					30.05 & 0.8334   &   \underline{33.50} &   {0.9197}
					&  31.93 & 0.8568   &  {\bf 37.02}  & {\bf 0.9619}
					\\ \hline
					4
					&   {25.95} & {0.7524}
					&   {30.45} & {0.8717}
					&   {33.68} & \underline{0.9536}
					&   {33.27} & {0.9287}
					&   {33.23} & {0.9280}
					&  {31.67}  & {0.8671} &
					30.00 & 0.8333   &   \underline{33.44}  &  {0.9187}
					&  31.83 & 0.8556   &  {\bf 37.03}  & {\bf 0.9621}
					\\ \hline
					5
					&   {25.99} & {0.7572}
					&   {30.54} & {0.8769}
					&   {33.18} & \underline{0.9490}
					&   {33.28} & {0.9309}
					&   {33.24} & {0.9302}
					&   {31.56} &  {0.8671} &
					30.02 & 0.8372   &   \underline{33.29}  &   {0.9203}
					&  31.76 & 0.8557  &  {\bf 36.84}  & {\bf 0.9626}
					\\ \hline
					6
					&   {26.03} & {0.7568}
					&   {30.55} & {0.8751}
					&   {32.59} & {0.8414}
					&   {33.24} & \underline{0.9305}
					&   {33.20} & {0.9300}
					&   {31.65} &  {0.8658} &
					30.16 & 0.8375   &   \underline{33.31}  &  {0.9175}
					&  31.82 & 0.8585   &  {\bf 37.21}  & {\bf 0.9618}
					\\ \hline
					7
					&   {25.31} & {0.7285}
					&   {29.75} & {0.8584}
					&   {32.18} & \underline{0.9335}
					&   {32.87} & {0.9257}
					&   {32.84} & {0.9251}
					&   {31.64} &  {0.8626} &
					31.69 & 0.8555   &   \underline{33.45}  &   {0.9164}
					&  31.69 & 0.8555   &  {\bf 36.95}  & {\bf 0.9583}
					\\ \hline
					8
					&   {25.60} & {0.7402}
					&   {30.28} & {0.8679}
					&   {31.97} & \underline{0.9280}
					&   {32.68} & {0.9222}
					&   {32.64} & {0.9216}
					&   {31.97} &  {0.8697} &
					30.67 & 0.8406   &  \underline{34.05}  &   {0.9233}
					&  32.00 & 0.8616   &  {\bf 37.27}  & {\bf 0.9610}
					\\ \hline
					9
					&   {25.95} & {0.7552}
					&   {30.52} & {0.8741}
					&   {31.85} & \underline{0.9261}
					&   {32.35} & {0.9139}
					&   {32.30} & {0.9129}
					&   {31.97} &  {0.8686} &
					30.57 & 0.8411   &   \underline{34.14}  &  {0.9240}
					&  32.01 &  0.8608   &  {\bf 37.51}  & {\bf 0.9610}
					\\ \hline
					10
					&   {25.96} & {0.7552}
					&   {30.55} & {0.8750}
					&   {30.71} & {0.9094}
					&   {29.68} & {0.8684}
					&   {29.63} & {0.8673}
					&   {31.86} &  {0.8662} &
					30.14 & 0.8360   &   \underline{34.19}  &  \underline{0.9242}
					&  31.99 & 0.8601   &  {\bf 37.35}  & {\bf 0.9606}
					\\ \hline
					average
					&   {25.78} & {0.7468}
					&   {30.32} & {0.8698}
					&   {32.63} & \underline{0.9383}
					&   {32.32} & {0.9130}
					&   {32.28} & {0.9122}
					&   {31.73} &  {0.8662} &
					30.34 & 0.8420   &   \underline{33.70}  &  {0.9208}
					&  31.85 & 0.8575   &  {\bf 37.13}  & {\bf 0.9611}
					\\ \hline
			\end{tabular}}
		}
	\end{center}

\end{table*}

\begin{figure}
	\centering
	\includegraphics[width=0.46\textwidth,height=0.3\textwidth]{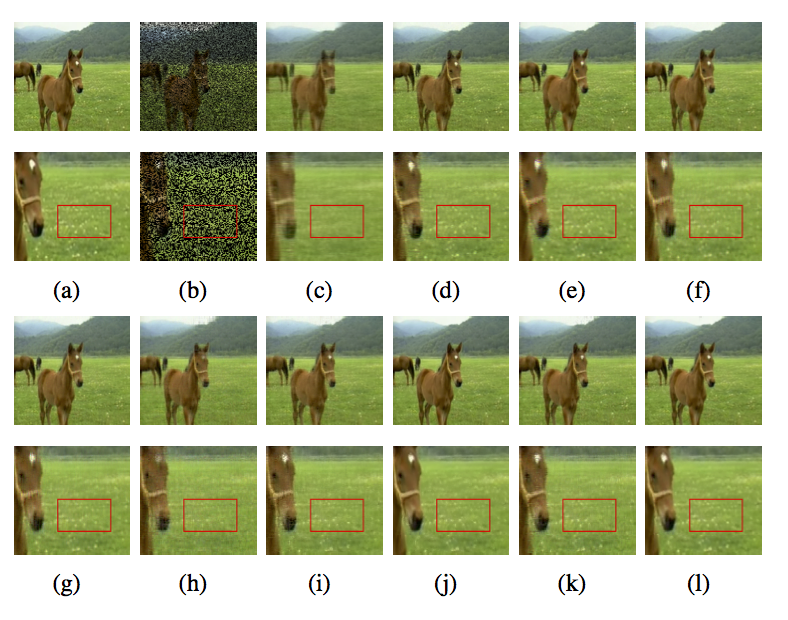}

%\vspace{-0.35cm}

\caption{ One frame (1st ans 3rd row) and their enlarged parts (2nd and 4th row) by ten methods with 80\% unobserved entries. (a) the original. (b)
	the observations.(c) t-SVD (Fourier). (d) t-SVD (data). (e) TNN. (f) OITNN-O. (g) OITNN-L. (h) LRQTC. (i) QMC. (j) TNSS-QMC. (k)  RQTC. ($\ell$) LRL-RQTC.}
	\label{fig:videohorsenan08nois00}
\end{figure}
\begin{figure}
	\centering
	\includegraphics[width=0.46\textwidth,height=0.3\textwidth]{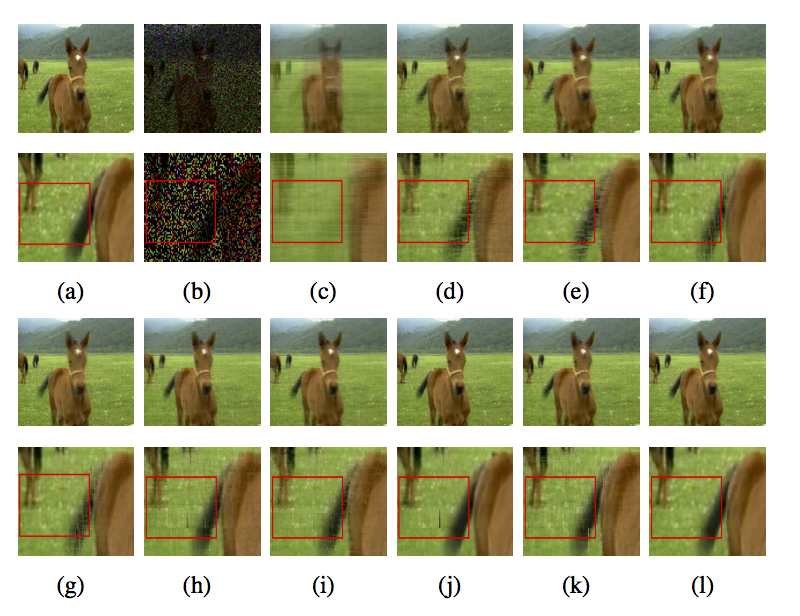}

%\vspace{-0.35cm}

\caption{ One frame (1st ans 3rd row) and their enlarged parts (2nd and 4th row) by ten methods  with 85\% unobserved entries. ((a) the original. (b)
	the observations. (c) t-SVD (Fourier). (d) t-SVD (data). (e) TNN. (f) OITNN-O. (g) OITNN-L. (h) LRQTC. (i) QMC. (j) TNSS-QMC. (k) RQTC. ($\ell$) LRL-RQTC.}
\label{fig:videohorsenan09nois00}
\end{figure}
\begin{table}
	
	\tabcolsep 0pt \caption{AVERAGE CPU TIME OF EACH FRAME OF THE RESTORED VIDEOS (UNIT:SECOND)}\label{cpu} \vspace*{-22pt}
	\begin{center}
		\tabcolsep=0.1cm
		\renewcommand{\arraystretch}{1.2}
		\def\temptablewidth{0.5\textwidth}
		{\rule{\temptablewidth}{0.5pt}}
		\begin{tabular*}{\temptablewidth}{@{\extracolsep{\fill}}c|cc|cc|cc}
			\multirow{2}*{Methods} & \multicolumn{2}{c|}{bird video }&\multicolumn{2}{c|}{flower video} & \multicolumn{2}{c}{horse video} \\
			&{(50\%,0)}&{(80\%,0)}&{(80\%,0)}&{(85\%,0)}&{(80\%,0)}&{(85\%,0)}\\ 
			\hline
			t-SVD(Fourier)& 8.22   &9.64   &6.61   &6.87   &6.67   &6.83\\
			t-SVD(data)   & 15.97  &18.58  &12.92  &12.59  &10.97  &11.01\\
			TNN           & \bf 5.05   &\bf 4.17   &\bf 4.54   &\bf 4.03   &\bf 2.86   &\bf 2.85\\
			OITNN-O       & 36.25  &35.57  &36.41  &18.30  &24.08  &12.50\\
			OITNN-L       & 37.02  &36.96  &36.98  &36.74  &24.02  &24.24\\
			QMC           & 96.12  &84.06  &88.79  &96.27  &80.11  &87.88\\
			LRQTC         & 19.61  &19.61  &19.51  &19.49  &19.52  &19.47\\
			TNSS-QMC      & 15.31  &16.05  &14.75  &14.78  &13.29  &12.18\\
			RQTC       & 22.74  &21.98  &20.65  &21.23  &19.67  &19.63\\
			LRL-RQTC      & 26.21  &24.62  &25.25  &26.92  &23.84  &22.98\\
			\hline
			
		\end{tabular*}
		{\rule{\temptablewidth}{0.5pt}}
	\end{center}			
\end{table}

	\section{Conclusion}\label{sec:con}
	In this article, we develop a new learning technology-based robust quaternion tensor completion model, LRL-RQTC.  Firstly, we divide the observed large quaternion tensor into smaller quaternion sub-tensors and parallelly act on these quaternion sub-tensors. Secondly, we use 2DQPCA-based classification function to learning low-rank information  of each slices and form them  into a low-rank quaternion matrix.  Then RQTC problem of each quaternion sub-tensors can be solved. The recommended technique focuses on keeping as many low-rank correlations among the local characteristics of color videos as feasible in order to preserve more related information under the framework.
	 A new  RQTC model is also proposed to solve RQTC problem by ADMM-based framework. We 
	 establish  conditions for quaternion tensor incoherence and exact recovery theory.  Numerical experiments on the established RQTC and LRL-RQTC models  demonstrate that our proposed models can inpaint a given missing and/or corrupted quaternion tensor better, maintaining a  low-rank structure for processing nature color videos  both effectively and efficiently. In order to make the most of the information between images and get a faster
	 and more efficient algorithm, we will strive to develop a  learning technology combined with other advanced technology for quaternion tensor decomposition in the future.

	\section*{Acknowledgment}

	This work is  supported  in part by the National Natural Science Foundation of China under grants 12171210, 12090011,   and 11771188;  
the Major Projects of Universities in Jiangsu Province (No. 21KJA110001);   
 and the Natural Science Foundation of Fujian Province of China grants 2020J05034.

	%We are grateful to the handing editor and the anonymous referees  for their useful suggestions and to Dr. Yongyong Chen and Dr. Jifei Miao for sharing  with us their MATLAB codes.

	\ifCLASSOPTIONcaptionsoff
	\newpage
	\fi

	
	
	%\bibliographystyle{IEEEtran}


\end{document}